\documentclass[english]{article}
\usepackage{geometry}
\usepackage[T1]{fontenc}
\usepackage[utf8]{inputenc}
\usepackage{bm}
\usepackage{amsmath}
\usepackage{amssymb} 
\usepackage{hyperref}
\hypersetup{
 colorlinks=true,
 citecolor=blue,
 filecolor=blue,
 linkcolor=blue,
 urlcolor=blue
}
\usepackage{xcolor,colortbl}
\definecolor{Gray}{gray}{0.85}
\usepackage{enumitem}

\usepackage{url}
\usepackage{float}
\usepackage{placeins}
\usepackage{longtable}
\usepackage{microtype}
\usepackage{graphicx}
\usepackage{subcaption}
\usepackage{booktabs}

\usepackage{mathtools}
\usepackage{amsthm}
\usepackage{algorithm}
\usepackage{algorithmic}
\usepackage[capitalize,noabbrev]{cleveref}
\usepackage{natbib}
\usepackage{bbm}

\usepackage{pifont}

\theoremstyle{plain}
\newtheorem{theorem}{Theorem}[section]
\newtheorem{proposition}[theorem]{Proposition}
\newtheorem{lemma}[theorem]{Lemma}

\theoremstyle{definition}

\newtheorem{assumption}[theorem]{Assumption}
\theoremstyle{remark}
\newtheorem{remark}[theorem]{Remark}

\newcommand{\norm}[1]{\left\lVert #1 \right\rVert}

\newcommand{\E}{\mathbb{E}}

\definecolor{lxs}{RGB}{200,0,0}

\usepackage{xspace}
\providecommand{\ul}[1]{\underline{#1}}
\newcommand{\RLPO}{{\sf LoRA-RLPO}\xspace}
\newcommand{\RLMO}{{\sf LoRA-RLMO}\xspace}

\newcommand{\mymid}{\,|\,} 

\usepackage{scalerel,stackengine}
\stackMath
\newcommand\reallywidehat[1]{%
\savestack{\tmpbox}{\stretchto{%
  \scaleto{%
    \scalerel*[\widthof{\ensuremath{#1}}]{\kern-.6pt\bigwedge\kern-.6pt}%
    {\rule[-\textheight/2]{1ex}{\textheight}}
  }{\textheight}%
}{0.5ex}}%
\stackon[1pt]{#1}{\tmpbox}%
}
\newcommand\reallywidecheck[1]{%
\savestack{\tmpbox}{\stretchto{%
  \scaleto{
    \scalerel*[\widthof{\ensuremath{#1}}]{\kern-.6pt\bigwedge\kern-.6pt}%
    {\rule[-\textheight/2]{1ex}{\textheight}}
  }{\textheight}%
}{0.5ex}}%
\stackon[1pt]{#1}{\scalebox{-1}{\tmpbox}}%
}

\title{Geometry-Preserving Orthonormal Initialization for\\
Low-Rank Adaptation in RLVR}

\author{
    Ruijia Zhang\thanks{Department of Applied Mathematics and Statistics, Johns Hopkins University, MD, USA.} \\
    JHU
    \and
    Jiacheng Zhu\thanks{Meta Superintelligence Labs.}  \\
    Meta
    \and 
    Hanqing Zhu\thanks{Department of Electrical and Computer Engineering, The University of Texas at Austin, TX, USA.} \\
    UT Austin
    \and
    Laixi Shi\thanks{Department of Electrical and Computer
Engineering, Johns Hopkins University, MD, USA.} \\
    JHU
}

\date{\today}

\begin{document}

\maketitle

\begin{abstract}
  Low-rank adaptation (LoRA) and its variants enable parameter-efficient fine-tuning of large language models under the supervised fine-tuning (SFT) paradigm.
  However, their efficacy and behavior under Reinforcement learning with verifiable rewards (RLVR) are less well understood. In particular, two structurally
  initialized LoRA variants, PiSSA and MiLoRA, which outperform standard LoRA under SFT, can underperform standard LoRA under RLVR and may even exhibit
  training instability. These observations suggest that how to initialize the low-rank matrices in RLVR remains unclear. In this work, we develop a
  theoretical analysis of LoRA in RLVR, showing that orthonormal initialization achieves the minimal gap between LoRA’s outcome and that of full fine-tuning.
  Guided by this insight, we propose geometry-preserving orthonormal initialization for low-rank adaptation in RLVR, leading to two new variants, \RLPO and
  \RLMO. Experiments on mathematical reasoning benchmarks show that the proposed orthonormal initialization stabilizes RLVR training and outperforms standard LoRA,
  contrasting with PiSSA and MiLoRA. Finally, our unified analysis for LoRA initialization also explains why PiSSA and MiLoRA can underperform in RLVR, which may be of independent
  interest. Code and checkpoints are publicly available at \href{https://github.com/Richard-ZZZ/geometry-preserving-orthonormal-init-rlvr}{the repository}.
  \end{abstract}

\tableofcontents

\section{Introduction}
\label{sec:intro}

Large language models (LLMs) \citep{NEURIPS2020_1457c0d6,touvron2023llamaopenefficientfoundation} are typically pretrained on large-scale dataset via next-token prediction \citep{NEURIPS2020_1457c0d6} and then fine-tuned on relatively smaller datasets to specialize for downstream applications. This paradigm has achieved remarkable success across diverse domains, including mathematical reasoning \citep{luo2025wizardmathempoweringmathematicalreasoning,azerbayev2024llemmaopenlanguagemodel}, code generation \citep{rozière2024codellamaopenfoundation,luo2024wizardcoder}, healthcare \citep{singhal2023large,chen2023meditron70bscalingmedicalpretraining}, and finance \citep{wu2023bloomberggptlargelanguagemodel,yang2025fingptopensourcefinanciallarge}. Because fine-tuning is far more accessible than pretraining a new LLM, it has attracted substantial interest in the community. While fine-tuning all parameters in an LLM (“full fine-tuning”) is natural, it is practically highly memory-intensive: fully fine-tuning even a 7B model can require over 100GB of GPU memory \citep{dettmers2023qlora}. This high resource demand limits accessibility for practitioners and motivates parameter-efficient fine-tuning (PEFT) methods, which update only a small subset of parameters while keeping the base model frozen \citep{houlsby2019parameter,li2021prefix}. Among them, Low-Rank Adaptation (LoRA) \citep{hu2022lora} is widely used due to its efficiency and ease of implementation. For any weight matrix $W_0 \in \mathbb{R}^{m \times n}$ in a pretrained model, LoRA parameterizes the update as $\Delta W^{\text{lora}}= BA$ with $B \in \mathbb{R}^{m \times r}$ and $A \in \mathbb{R}^{r \times n}$, where $r$ is much smaller than $\min(m,n)$. This significantly reduces the number of parameters under training, while still yielding a dense matrix $W_0 + \Delta W^{\text{lora}}$ at inference time.

Beyond supervised fine-tuning (SFT), a widely used LLM fine-tuning paradigm that trains on high-quality question--response pairs, reinforcement learning with verifiable rewards (RLVR) has recently emerged as a pivotal paradigm, proving effective across tasks such as mathematical reasoning and coding~\citep{Guo_2025,shao2024deepseekmathpushinglimitsmathematical}. RLVR uses rule-based feedback (e.g., answer correctness) instead of learned reward models. However, RLVR incurs substantially higher memory costs than SFT, as it requires keeping a reference model in memory to compute KL divergence~\citep{ziegler2020finetuninglanguagemodelshuman,zhou2024understandingalleviatingmemoryconsumption} and storing multiple responses per prompt for group-based advantage estimation~\citep{shao2024deepseekmathpushinglimitsmathematical}. This makes LoRA and its memory-efficient variants particularly attractive for RLVR, especially given that LoRA has already shown strong potential in this setting, matching full fine-tuning in certain cases~\citep{schulman2025lora}.

Despite this progress, the behavior of LoRA and its structural variants under RLVR remains less understood than under SFT, limiting further advances in low-rank fine-tuning for RL. In particular, how to initialize the low-rank matrices $B$ and $A$ is increasingly unclear in light of recent observations. PiSSA~\citep{meng2024pissa} and MiLoRA~\citep{wang2025milora}, two LoRA variants that improve performance and accelerate convergence in SFT, can underperform standard LoRA under RLVR and may even exhibit training instability~\citep{yin2025evaluatingparameterefficientmethods}. Both methods initialize $B$ and $A$ via the singular value decomposition (SVD) of pretrained weights, but in opposite directions: PiSSA uses the top-$r$ principal singular directions, while MiLoRA targets the bottom-$r$ tail directions.  In addition, prior work suggests that, due to the KL constraint, RLVR updates are encouraged to stay close to the reference policy ~\citep{wu2026invisibleleashrlvrescape,shenfeld2025rlsrazoronlinereinforcement} and may favor off-policy subspaces that differ from those preferred by SFT ~\citep{zhu2025pathtakenrlvrprovably}. This discrepancy between RLVR and SFT likely reflects their distinct optimization dynamics. Consequently, LoRA design principles developed for SFT are no longer guaranteed to transfer to RLVR, leaving the appropriate initialization and subspace choice in RLVR an open question.
In this work, we focus on:
\begin{center}
    {\em What initialization is effective for low-rank adaptation in RLVR fine-tuning?}
\end{center}

To this end, we provide a rigorous analysis demonstrating that by initializing $B = 0$ in accordance with standard LoRA, orthonormal initialization for $A$ is potentially optimal and yields superior performance in practice. Our primary contributions are as follows:
\begin{itemize}
\vspace{-5pt}
\setlength\itemsep{0em}
    \setlength\parskip{0.5em}
    \setlength\parsep{0em}
    \item \textbf{Orthonormal initialization towards optimal.} To understand the behavior of LoRA, we provide a theoretical analysis of LoRA's optimization dynamics, showing that orthonormal initialization of $A$ with $B=0$ minimizes the gap between LoRA and full fine-tuning (Theorem ~\ref{thm:lora_error}).
    
    \item \textbf{Geometry-preserving orthonormal initialization.}
  Motivated by this result, we propose \RLPO and \RLMO, SVD-based initializations for $A$ that remain orthonormal while preserving geometric information from pretrained weights. Both methods outperform standard LoRA in RLVR in the conducted experiments.

    \item \textbf{Insights into LoRA variants' failures in RLVR.} Our analysis framework offers a unified explanation for the instability of SVD-based methods such as PiSSA and MiLoRA in RLVR. Their failures stem from two coupled factors: \textbf{subspace geometry}, which accelerates updates along specific directions, and \textbf{singular value scaling}, which amplifies update magnitudes. Together, these induce aggressive optimization trajectories that rapidly violate the implicit KL constraint, destabilizing training regardless of whether principal or minor singular directions are targeted.
   
\end{itemize}

\begin{figure*}[t]
    \centering
    \includegraphics[width=1.0\linewidth]{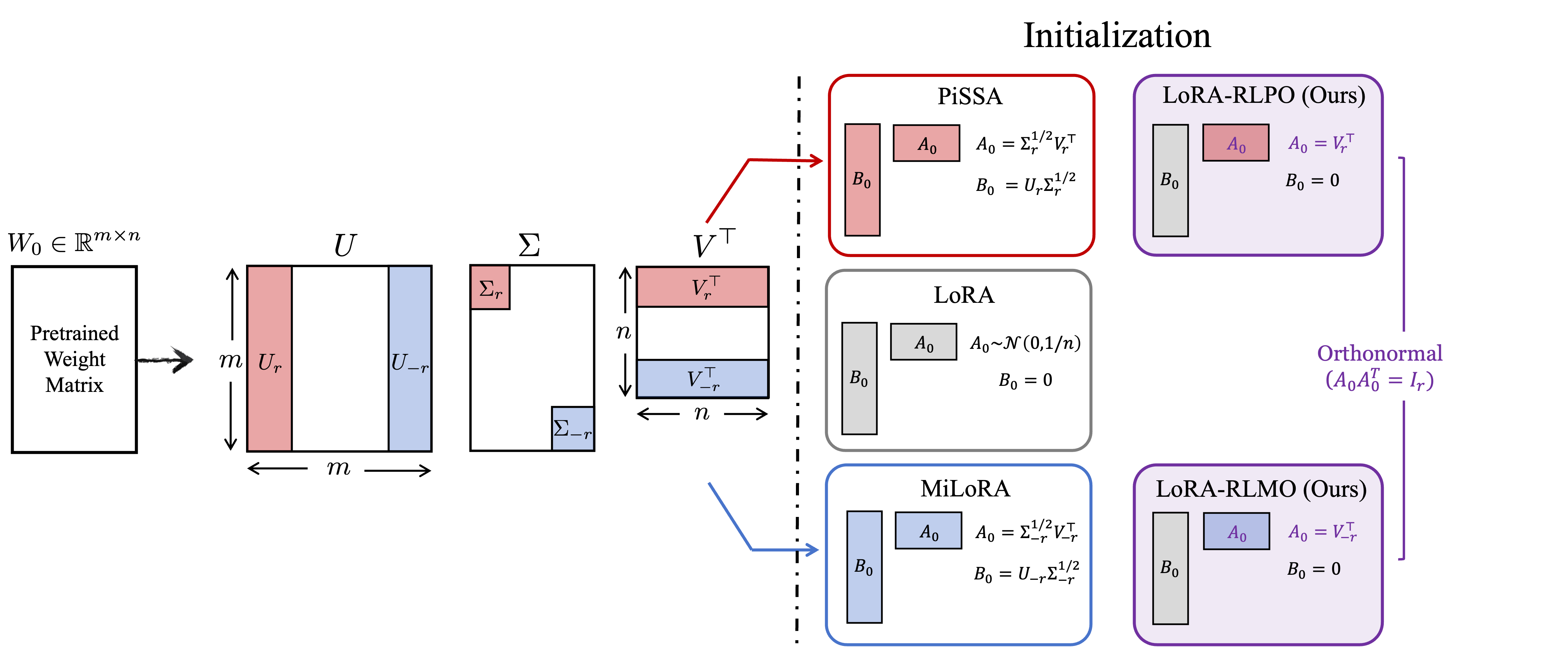}
    \caption{Comparison of LoRA initialization strategies. LoRA uses random Gaussian $A_0$ with $B_0 = 0$. PiSSA and MiLoRA initialize both adapters from the principal and minor singular components of $W_0$, respectively, with $B_0 \neq 0$ and singular value scaling. Our proposed methods, \textbf{\RLPO} and \textbf{\RLMO}, initialize orthonormal $A_0$ from the principal and minor right singular vectors with $B_0 = 0$.}
    \label{fig:visual}
    \vspace{-10pt}
\end{figure*}

\section{Related Works}
\label{sec:related}

\paragraph{Low-rank adaptation (LoRA) and its variants.}
Among parameter-efficient fine-tuning (PEFT) methods, LoRA~\citep{hu2022lora} and its variants have become a popular class, parameterizing weight updates as a product of two low-rank matrices while keeping the base model frozen. Numerous variants have been proposed to improve upon standard LoRA. One line modifies the optimization process: AdaLoRA~\citep{zhang2023adaloraadaptivebudgetallocation} adaptively allocates rank across layers based on importance scores; LoRA+~\citep{pmlr-v235-hayou24a} uses different learning rates for $A$ and $B$; DoRA~\citep{liu2024dora} decomposes updates into magnitude and direction components; rsLoRA~\citep{kalajdzievski2023rankstabilizationscalingfactor} adjusts the scaling factor to stabilize training at higher ranks; and VeRA~\citep{kopiczko2024vera} shares frozen random matrices across layers to further reduce parameters. Another line improves LoRA initialization beyond the default random scheme~\citep{hu2022lora,pmlr-v235-hayou24a}. SVD-based methods have drawn particular attention: PiSSA~\citep{meng2024pissa} initializes adapters using principal singular components, while MiLoRA~\citep{wang2025milora} uses minor components. These methods achieve faster convergence and improved performance in supervised fine-tuning, but recent evaluations show that they underperform standard LoRA and exhibit instability in RLVR~\citep{yin2025evaluatingparameterefficientmethods}. Our work follows this line and focuses on understanding and addressing this gap through geometry-preserving orthonormal initialization.

\paragraph{Orthonormality in LoRA.}
Several prior works have explored the role of orthonormality in LoRA, primarily in the context of supervised fine-tuning. \citet{zhu2024asymmetrylowrankadaptersfoundation} investigate the asymmetry between the two LoRA matrices, showing that $A$ extracts features from inputs while $B$ maps these features to outputs; they further demonstrate that fixing $A$ as a random orthonormal matrix and training only $B$ outperforms standard LoRA.  OLoRA~\citep{büyükakyüz2024oloraorthonormallowrankadaptation} uses Qrthogonal-Right triangular(QR) decomposition to initialize both LoRA matrices with orthonormal bases derived from the pretrained weights, achieving faster convergence on SFT tasks. From a complementary perspective, OFT~\citep{qiu2023controlling} and BOFT~\citep{liu2024parameter} enforce orthogonality of weight updates throughout training, rather than only at initialization. However, these studies are largely empirical and confined to the SFT regime, whose learning dynamics differ substantially from those of RLVR. Our work provides the first theoretical explanation for why orthonormal initialization improves LoRA in RLVR, showing that it enables LoRA to more closely track the trajectory of full fine-tuning.

\paragraph{LoRA for Reinforcement learning with verifiable rewards.}
While LoRA has been extensively studied in supervised fine-tuning~\citep{hu2022lora,liu2024dora,kalajdzievski2023rankstabilizationscalingfactor,pmlr-v235-hayou24a}, its behavior under RL-based fine-tuning remains far less understood, despite its widespread adoption for memory-efficient PPO and GRPO training on consumer hardware~\citep{santacroce2023efficientrlhfreducingmemory,Guo_2025,shao2024deepseekmathpushinglimitsmathematical}. \citet{zhu2025pathtakenrlvrprovably} provides theoretical analysis showing that RLVR updates favor off-principal directions, in contrast to SFT which targets principal components, suggesting that methods designed for SFT may not transfer directly to RLVR. \citet{yin2025evaluatingparameterefficientmethods} systematically evaluate PEFT methods under RLVR and find that SVD-based initializations such as PiSSA and MiLoRA underperform standard LoRA and exhibit training instability. Despite this progress, the appropriate initialization and subspace choice for LoRA in RLVR remains unsettled. Our work addresses this gap by providing a unified theoretical framework that explains both failure modes and showing that, when paired with the standard LoRA choice of $B=0$, geometry-preserving orthonormal initialization of $A$ offers a principled practical solution.

\section{Background}
\label{sec:background}

In this section, we formalize the fine-tuning problem for LLMs. Consider an LLM parameterized by $\theta = \{W^{(\ell)}\}_{\ell=1}^{L}$, the collection of weight matrices across its $L$ layers (e.g., fully connected and attention layers). Without loss of generality and a slight abuse of notation,  we focus on a single weight matrix $W$ in the following discussion. Full fine-tuning optimizes the model parameters $\theta$ by updating all weight matrices. Specifically, for any pretrained weight matrix $W_0 \in \mathbb{R}^{m \times n}$, full fine-tuning learns $W = W_0 + \Delta W^{\mathrm{full}}$, where the update $\Delta W^{\mathrm{full}} \in \mathbb{R}^{m \times n}$ is unconstrained, so as to minimize a task-specific loss $\mathcal{L}(\theta) = \mathcal{L}(\{W^{(\ell)}\}_{\ell=1}^{L})$. This approach requires storing full gradients and optimizer states for all weight matrices, which can become prohibitive for large-scale models.

\subsection{Low-rank adaptation (LoRA) and variants}
\label{sec:lora_variants}
We first review the LoRA algorithm \citep{hu2022lora}. Consider a pretrained weight matrix $W_0 \in \mathbb{R}^{m\times n}$, LoRA parameterizes the weight update as
\begin{equation}
W = W_0 + \Delta W^{\text{lora}}, \qquad \text{with} \quad \Delta W^{\text{lora}} = BA,
\label{eq:lora}
\end{equation}
where $B \in \mathbb{R}^{m\times r}$ and $A \in \mathbb{R}^{r\times n}$ are low-rank matrices with rank $r$ much smaller than $m$ and $n$ ($r \ll \min\{m,n\}$). Consequently, the optimization of $\Delta W^{\text{lora}}$ is restricted to a low-rank subspace of $\mathbb{R}^{m\times n}$. The initialization is set as follows:
\begin{equation}
    B_0 = 0_{m \times r}, \quad A_0 \sim \mathcal{N}\left(0, \tfrac{1}{n}\right)^{r \times n}.
    \label{eq:lora_init}
\end{equation}

\paragraph{SVD-based initialization variants.}
Beyond the initialization in \eqref{eq:lora_init}, many prior works propose alternative initializations for low-rank fine-tuning. We describe two representative SVD-based variants below.
Let $W_0 = U\Sigma V^\top$ be the singular value decomposition, where $U \in \mathbb{R}^{m\times k}$, $\Sigma=\mathrm{diag}(\sigma_1,\ldots,\sigma_k)$ with $\sigma_1 \ge \sigma_2 \ge \cdots \ge \sigma_k > 0$, and $V \in \mathbb{R}^{n\times k}$. Here, $k=\mathrm{rank}(W_0)$, i.e., the number of positive singular values of $W_0$.

\textbf{PiSSA} \citep{meng2024pissa} initializes with top-$r$ principal components as follows:
\begin{equation}
    B_0 = U_r \Sigma_r^{1/2}, \quad A_0 = \Sigma_r^{1/2} V_r^\top
    \label{eq:pissa}
\end{equation}
where $U_r$ and $V_r$ denote the first $r$ columns of $U$ and $V$, respectively, and $\Sigma_r$ is the $r \times r$ diagonal matrix of the top $r$ singular values.

\textbf{ MiLoRA} \citep{wang2025milora} initializes using the bottom-$r$ minor components:
\begin{equation}
    B_0 = U_{-r} \Sigma_{-r}^{1/2}, \quad A_0 = \Sigma_{-r}^{1/2} V_{-r}^\top,
    \label{eq:milora}
\end{equation}
where $U_{-r}$ and $V_{-r}$ denote the last $r$ columns of $U$ and $V$, respectively, and $\Sigma_{-r}$ is the $r \times r$ diagonal matrix of the bottom $r$ singular values.

Both methods then replace each pretrained matrix $W_0$ with the residual $W_0 - B_0 A_0$, which is kept frozen, and optimize only $BA$ thereafter. Equivalently, the effective weight is parameterized as $W = (W_0 - B_0 A_0) + BA$, ensuring $W = W_0$ at initialization regardless of which singular components are used.

\subsection{Finetuning paradigms: SFT vs.\ RLVR}\label{sec:finetuning_paradigms}

Besides the parameter-update setting, we now introduce two widely used fine-tuning frameworks and their corresponding objective functions. 
\paragraph{Supervised fine-tuning (SFT).} Consider a supervised dataset $\mathcal{D} = \{(q_i, a_i^\star)\}_{i=1}^N$
consisting of many input-output pairs, where \(q_i\) denotes a prompt or instruction and \(a_i^\star\) is the corresponding ground-truth response. The SFT process minimizes cross-entropy against ground-truth labels as follows:
\begin{equation*}
    \mathcal{L}_{\text{SFT}}(\theta) := -\E_{(q, a^*) \sim \mathcal{D}}\left[\log \pi_\theta(a^* | q)\right].
\end{equation*}
SFT imposes no explicit constraint on the weight movement, allowing the parameters to drift arbitrarily far from $W_0$. 

\paragraph{Reinforcement learning with verifiable rewards (RLVR).}
RLVR fine-tunes large language models with RL using automatically
verifiable, rule-based rewards $R$ (e.g., exact-match correctness on
math or code), thereby eliminating the need for a learned reward model.
In this work, we focus on DAPO~\citep{yu2026dapo} because it is a representative state-of-the-art RLVR algorithm for long-CoT reasoning and, unlike GRPO, removes the explicit KL penalty to the reference policy, making it a cleaner setting for isolating and understanding the factors that affect RLVR training stability.

Moreover, DAPO and related clipped-policy RLVR algorithms are widely used for evaluating LoRA-style adaptation, making them a natural testbed for analyzing the effect of low-rank initialization~\citep{yin2025evaluatingparameterefficientmethods}.
Specifically, DAPO samples a group of outputs $\{o_i\}_{i=1}^{G}$ for
each question $q$ paired with answer $a$, and updates the policy by
optimizing the following clipped importance-ratio objective:
\begin{equation}\label{eq:rlvr}
\begin{aligned}
\mathcal{L}^{\mathrm{DAPO}}(\theta^+)
&= \mathbb{E}_{(q,a)\sim\mathcal{D},\,\{o_i\}\sim\pi_{\theta}(\cdot\mid q)}
\Bigg[
\frac{1}{\sum_{i=1}^{G}|o_i|}
\sum_{i=1}^{G}\sum_{t=1}^{|o_i|} \\
& \qquad\min\!\Big(r_{i,t}(\theta^+)\hat{A}_{i,t},\,
\operatorname{clip}\!\big(r_{i,t}(\theta^+),1-\epsilon_{\mathrm{low}},1+\epsilon_{\mathrm{high}}\big)\hat{A}_{i,t}\Big)
\Bigg], \\
&\text{s.t.}\quad
0 < \big|\{o_i \mid \mathsf{is\_equivalent}(a,o_i)\}\big| < G,
\end{aligned}
\end{equation}
where $r_{i,t}(\theta^+)
= \frac{\pi_{\theta^+}(o_{i,t}\mid q,o_{i,<t})}
{\pi_{\theta}(o_{i,t}\mid q,o_{i,<t})}, \hat{A}_{i,t}
= \frac{R_i-\operatorname{mean}(\{R_i\}_{i=1}^{G})}
{\operatorname{std}(\{R_i\}_{i=1}^{G})}.$
A key ingredient of DAPO is the clipped importance ratio, which constrains
$r_{i,t}(\theta^+)$ to
$[1-\epsilon_{\mathrm{low}},\,1+\epsilon_{\mathrm{high}}]$ and thereby
implicitly limits the policy drift between $\pi_{\theta^+}$ and
$\pi_\theta$.

\paragraph{Training stability demands constrained KL divergence.}
The conservative-update principle implemented by DAPO's clipped importance ratio is not unique to DAPO: it underlies many popular policy-gradient algorithms (e.g., TRPO~\citet{pmlr-v37-schulman15} and PPO~\citet{schulman2017proximalpolicyoptimizationalgorithms}) as well as RLVR popular variants such as GRPO~\citep{shao2024deepseekmathpushinglimitsmathematical}. By preventing the updated policy from deviating too far from the current one in a single step, the clipping mechanism can be viewed as a surrogate method to constrain the following KL divergence between $\pi_{\theta^+}$ and $\pi_\theta$~\citep{10.5555/645531.656005} in a safe region:
\begin{equation}\label{eq:kl_constraint}
  D_{\mathrm{KL}}(\pi_{\theta^+}\,\|\,\pi_\theta)
  =
  \mathbb{E}_{q\sim\mathcal{D},\, y\sim\pi_{\theta^+}(\cdot\mid q)}
  \left[
  \sum_{t}
  \log
  \frac{
  \pi_{\theta^+}(y_t\mid q,o_{<t})
  }{
  \pi_{\theta}(y_t\mid q,o_{<t})
  }
  \right].
\end{equation}
Consequently, an excessively large $D_{\mathrm{KL}}(\pi_{\theta^+}\,\|\,\pi_\theta)$ may violate this implicit trust-region constraint and can cause performance degradation or training collapse, since the surrogate objective is no longer guaranteed to be a lower bound of the true reward objective~\citep{10.5555/645531.656005}. We therefore adopt $D_{\mathrm{KL}}(\pi_{\theta^+}\,\|\,\pi_\theta)$ as a key diagnostic for the training stability of RLVR, and report its sample estimate over response tokens in the empirical analyses that follow.

\section{Instability of SVD-Based LoRA Initializations in RLVR}
Noting that the behavior of LoRA variants in RLVR remains underexplored,
limiting further advances in low-rank fine-tuning for many tasks such as reasoning. In this work, we focus on studying the \emph{initialization} module of LoRA
for RLVR, focusing on two prominent variants (PiSSA and
MiLoRA) that have demonstrated strong performance in supervised
fine-tuning (SFT). While both PiSSA and MiLoRA outperform standard LoRA in SFT, they instead underperform standard LoRA under RLVR and even exhibit clear \ul{training collapse}, as shown in Figure~\ref{fig:Pissa/milora} (left). To understand the source of this failure, we monitor the introduced KL divergence in \eqref{eq:kl_constraint} and the Frobenius norm of the gradient for training stability, as shown in Figure~\ref{fig:Pissa/milora} (middle and right). 
 Both PiSSA and MiLoRA incur substantially larger gradient norms and higher cumulative KL divergence than standard LoRA throughout training, indicating markedly unstable optimization processes. These observations raise a natural question: {\em
Why do LoRA initialization principles that succeed in SFT, such as those of PiSSA and MiLoRA, break down under RLVR?}

\begin{figure}
    \centering
    \includegraphics[width=1\linewidth]{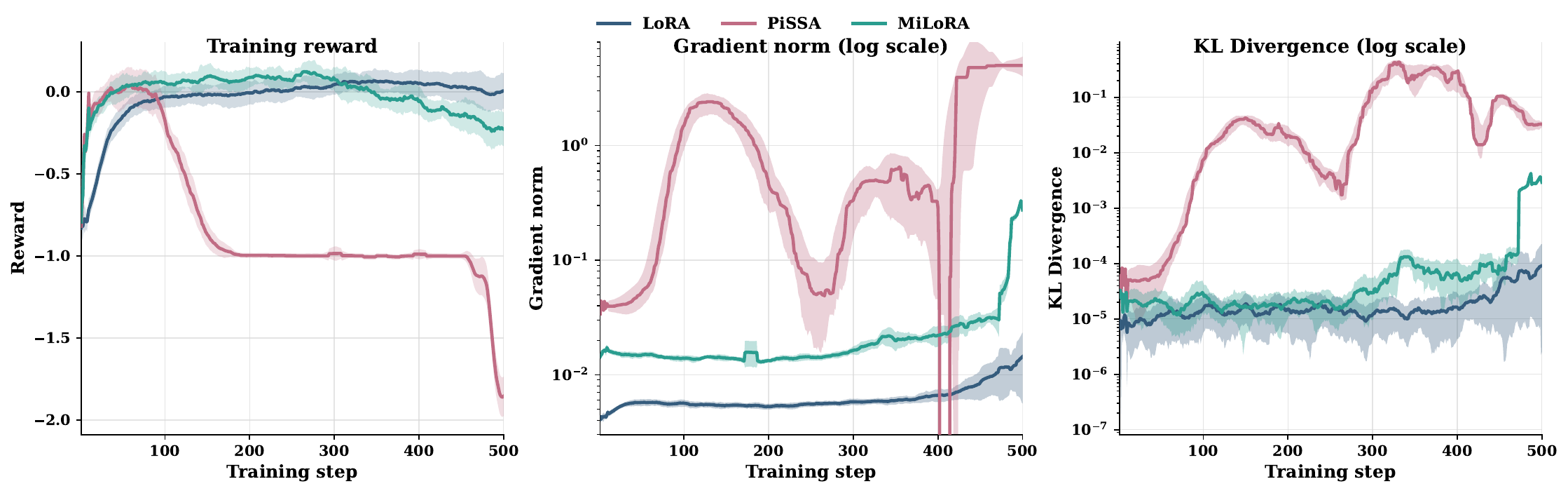}
    \caption{Training dynamics of RLVR via DAPO on benchmark DAPO-MATH. Left: Training reward comparison. Middle: The aggregate Frobenius norm of gradients over trainable parameters. Right: KL divergence during training. Both PiSSA and MiLoRA exhibit training reward collapse, higher gradient norm and KL divergence than standard LoRA.}
    \label{fig:Pissa/milora}
    \vspace{-10pt}
\end{figure}

\paragraph{Initialization of LoRA largely governs final optimization outcome.}
To answer the question, we begin by analyzing how the initialization step shapes the
subsequent optimization trajectory and the resulting fine-tuned
model. To this end, we visualize the post-training update
distribution and cumulative energy after RLVR fine-tuning in
Figure~\ref{fig:svd_style_lora_pissa_milora}. Recall that PiSSA
(cf.~\eqref{eq:pissa}) and MiLoRA (cf.~\eqref{eq:milora}) initialize
LoRA in the top and bottom singular-vector subspaces, respectively,
with singular value scaling through $\Sigma_r^{1/2}$ or
$\Sigma_{-r}^{1/2}$. As shown in
Figure~\ref{fig:svd_style_lora_pissa_milora}, the update magnitudes
and final energy distributions of the fine-tuned outputs closely
mirror these initialized subspaces: PiSSA produces substantially
larger updates along the top singular-vector directions, MiLoRA
concentrates its updates along the bottom singular-vector
directions, and in both cases the fine-tuned outcomes retain high
cumulative energy in the corresponding spectral regions. This
demonstrates that LoRA initialization does \emph{not} merely set the
optimization starting point, but rather it intrinsically governs the entire
optimization trajectory and, consequently, the final outcome.

\begin{figure*}[t]
      \centering
\includegraphics[width=0.95\linewidth]{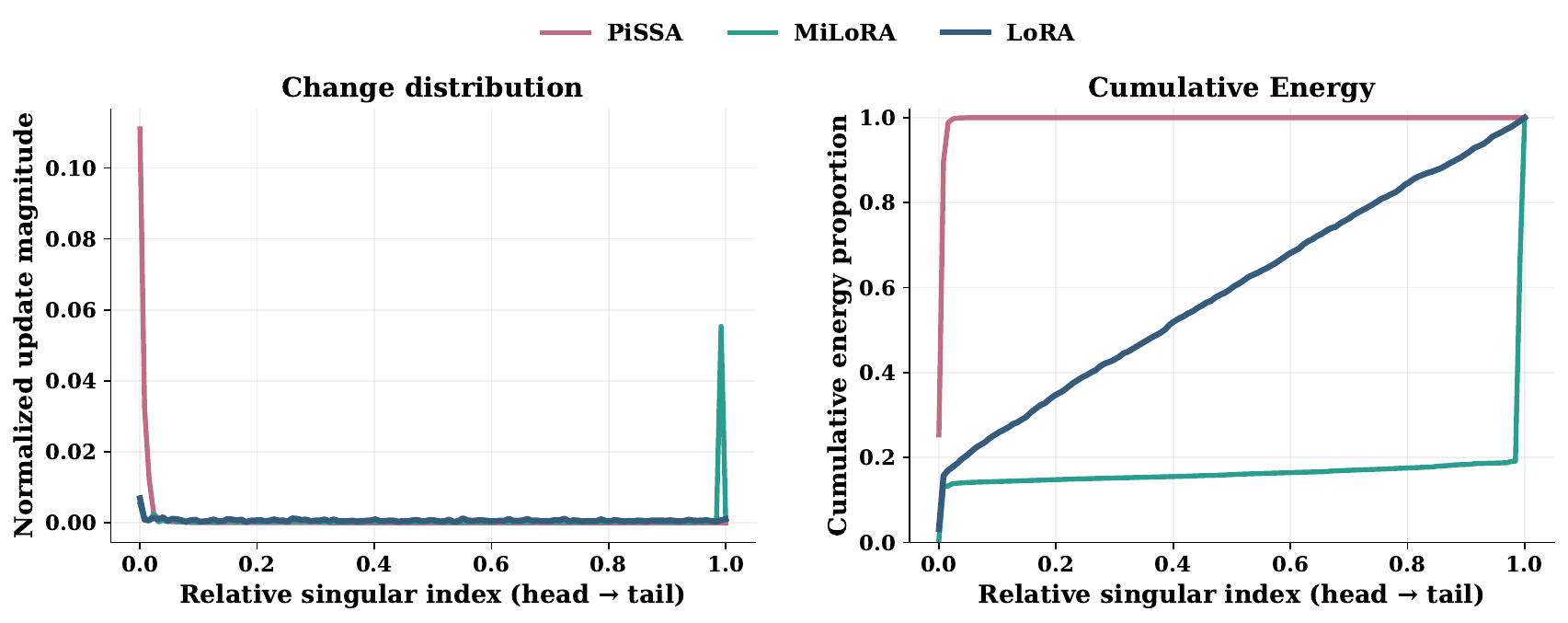}
 \caption{\textbf{SVD-aligned update distribution and cumulative energy after RLVR training.}
  For each method, we analyze the trained LoRA update
  \(\Delta W = \frac{\alpha}{r}BA\) on the query projection and attention output projection
  matrices from Transformer layers 0, 14, and 27 of the 28-layer model, where each projection
  weight has size \(W \in \mathbb{R}^{1536 \times 1536}\). Given the singular value decomposition
  of the frozen pretrained weight \(W = U\Sigma V^\top\), we measure the update magnitude along
  the pretrained singular directions as \(c_i = |u_i^\top \Delta W v_i|\), and normalize it by
  \(p_i = c_i / \sum_j c_j\). The left panel plots \(p_i\) over the relative singular-mode index,
  ordered from the largest to smallest singular values of \(W\), while the right panel plots the
  corresponding cumulative energy \(\sum_{j \leq i} c_j^2 / \sum_j c_j^2\). Curves are averaged
  over the six analyzed projection matrices. The distinct patterns indicate that initialization
  affects not only the starting point of optimization, but also the spectral structure of the
  updates learned during RLVR.}
\label{fig:svd_style_lora_pissa_milora}
  \end{figure*}

\subsection{Instability sources of LoRA family in RLVR.}
Another key observation from
Figure~\ref{fig:svd_style_lora_pissa_milora} is that, after
geometry-informed initialization, the final learned updates remain
highly concentrated along the predefined singular directions,
producing aggressively amplified weight changes. Such directionally
intensified updates likely drive the large gradient norms and rapid
KL growth observed under RLVR in
Figure~\ref{fig:Pissa/milora}, which in turn contribute to the
performance collapse of PiSSA and MiLoRA. To pinpoint the source of
this instability, we disentangle three potential factors:
(1) the learning rate during optimization, and two arising from
initialization, namely, (2) the selected singular subspace, which
determines the update direction, and (3) the singular value
scaling, which amplifies updates along that direction.

\begin{figure}[t]
    \centering
\includegraphics[width=1\linewidth]{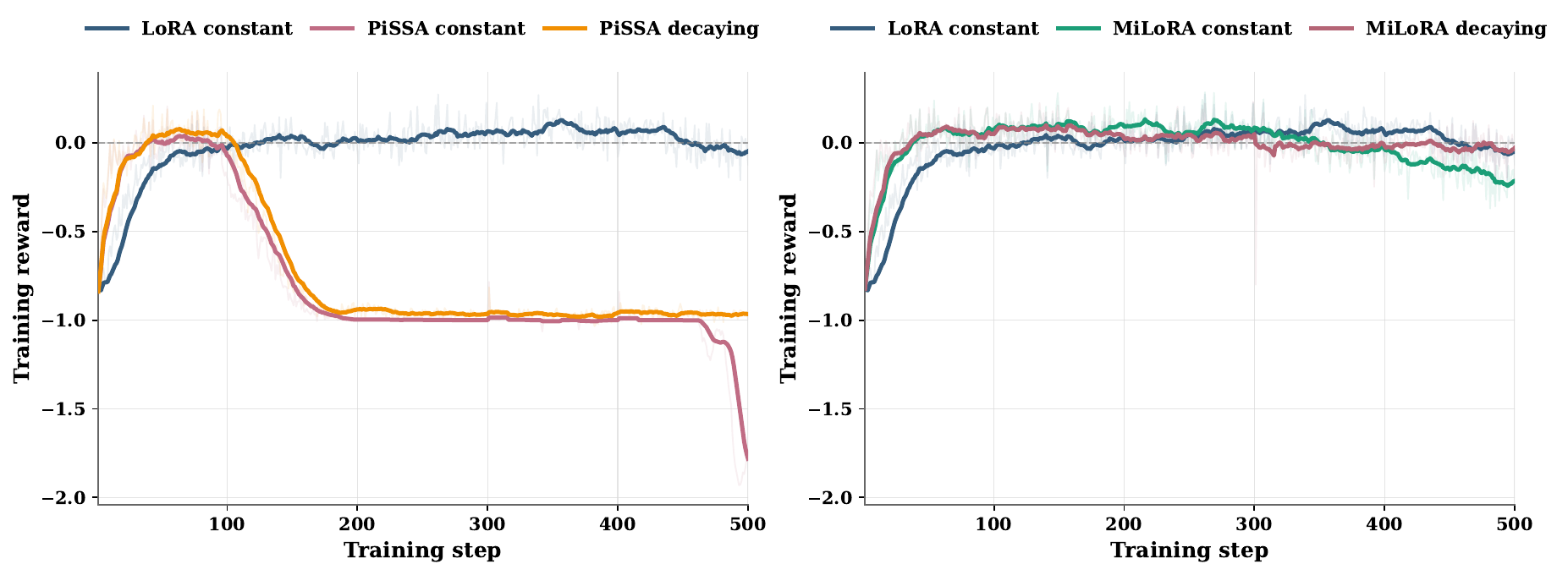}
   \caption{\textbf{Ablation on learning-rate decay for PiSSA and MiLoRA under RLVR training.}                           
   The left panel compares PiSSA with a constant learning rate and a cosine-decaying schedule, together with the LoRA       
 constant-learning-rate baseline; the right panel shows the analogous comparison for MiLoRA.
   Both methods benefit from slower optimization, confirming that enlarged effective updates are a primary source of     
 instability in RLVR training.                            
   Nevertheless, while MiLoRA with learning-rate decay approaches the stable behavior of standard LoRA, PiSSA still      
 exhibits significant late-stage collapse, pointing to an inherently more hazardous optimization direction.} 
    \label{fig:lr decay}
    \vspace{-10pt}
\end{figure}

\paragraph{Learning rate partially controls update magnitude.}
To assess the role of the learning rate in driving instability, we
conduct an ablation comparing constant and decaying schedules,
shown in Figure~\ref{fig:lr decay}. Slowing optimization through
learning rate decay offers partial mitigation for both PiSSA and
MiLoRA: as Figure~\ref{fig:lr decay} reveals, MiLoRA with a decaying
learning rate approaches the stable behavior of standard LoRA.
PiSSA, however, still suffers severe collapse in the later phases of
training, indicating that learning rate decay alone does not resolve
the underlying instability. This points to an inherently vulnerable
optimization landscape for PiSSA: once updates are initialized and
accelerated along
its geometry-informed directions, the optimization trajectory tends to exit the stable optimization regime.

\paragraph{Singular value scaling destabilizes training across singular subspaces.} 
Beyond the choice of subspace, singular value scaling is another
potentially critical factor for optimization stability. To isolate
its effect from subspace selection, we use \textbf{OLoRA} \citep{büyükakyüz2024oloraorthonormallowrankadaptation}  as a controlled
baseline. OLoRA initializes with top-$r$ principal components as follows:
\begin{equation}\label{eq:olora}
    B_0 = U_r, \quad A_0 = V_r^\top.
\end{equation}

Since OLoRA and PiSSA target the same principal subspace
of $W_0$, any difference in stability can be attributed to the
\emph{singular value scaling} inherent to PiSSA. We investigate this
effect from both theoretical and empirical perspectives.

\noindent{\em Theoretical analysis.}
We first present a theorem that quantifies how PiSSA's initialization potentially
induces excessive update magnitudes compared to OLoRA.

\begin{theorem}[PiSSA Gradient Amplification]\label{thm:pissa_amplification}
The first-step weight updates $\Delta W^{\text{PiSSA}}_1$ of PiSSA and $\Delta W^{\text{OLoRA}}_1$ of OLoRA satisfy:
\begin{equation}
    \frac{\|\Delta W^{\text{PiSSA}}_1\|_F}{\|\Delta W^{\text{OLoRA}}_1\|_F} \ge \sigma_r,
\end{equation}
where $\sigma_r$ is the $r$-th largest singular value of $W_0$.
\end{theorem}

The proof is deferred to Appendix~\ref{app:proof_pissa_olora_amplification}. Although PiSSA and OLoRA share the same principal subspace, PiSSA scales each retained singular mode by the corresponding singular value, amplifying the update norm by at least $\sigma_r$ to leading order. For pretrained LLMs, singular values typically follow a heavy-tailed distribution with $\sigma_r \gg 1$ at moderate rank $r$ (see Figure~\ref{fig:singular_values} in Appendix). Theorem~\ref{thm:pissa_amplification} therefore implies that PiSSA inflates weight updates by a substantial factor relative to OLoRA, increasing the risk of exceeding the implicit KL budget in RLVR.

To proceed, we invoke a result from \citet{zhu2025pathtakenrlvrprovably}, which shows that enforcing a KL constraint
\(D_{\mathrm{KL}}(\pi_{\theta^{+}} \,\|\, \pi_\theta)\) keeps the updated policy close to the current policy and, in turn, bounds the magnitude of the corresponding weight change.

\begin{theorem}[KL constraint implies weight bound {\citep[Gate I]{zhu2025pathtakenrlvrprovably}}]
\label{prop:kl_leash}
Assume $\log \pi_\theta$ is $C^3$, 
where $F(\theta)$ denotes the Fisher information matrix \footnote{
Here $C^3$ means having continuous derivatives up to order $3$.
The Fisher information matrix is defined as
$F(\theta)=\mathbb{E}_{x\sim \mathcal{D},\,y\sim \pi_\theta(\cdot|x)}
[\nabla_\theta \log \pi_\theta(y|x)\nabla_\theta \log \pi_\theta(y|x)^\top]$.
}.
Consider a single-step update of the model parameters from $\theta$ to $\theta^{+}$, so that each weight matrix $W \subset \theta$ is updated to $W + \Delta W$. Suppose that the KL divergence term in \eqref{eq:kl_constraint} $D_{\mathrm{KL}}(\pi_{\theta^{+}} \,\|\, \pi_\theta) \le K$ and that, on the update subspace, $F(\theta) \succeq \mu I$ for some $\mu > 0$. Then, for $K$ sufficiently small, every weight block $W \subset \theta$ satisfies $\|\Delta W\|_{F} \;\le\; \sqrt{\frac{2K}{\mu}}\,\bigl(1+o(1)\bigr).$
\end{theorem}
Theorem~\ref{prop:kl_leash} shows that when the KL divergence is constrained by $K$, a corresponding bound is imposed on every weight update, $\|\Delta W\|_F \lesssim \sqrt{2K/\mu}$. Consequently, an excessively large $\|\Delta W\|_F$ will drive $D_{\mathrm{KL}}(\pi_{\theta^{+}} \,\|\, \pi_\theta)$ beyond the budget $K$, potentially violating the conservative-update requirement and destabilizing training (see Section~\ref{sec:finetuning_paradigms}). As Theorem~\ref{thm:pissa_amplification} establishes, PiSSA's singular-value scaling amplifies the weight update $\|\Delta W^{\text{PiSSA}}_1\|_F$ along the principal spectral directions relative to OLoRA's $\|\Delta W^{\text{OLoRA}}_1\|_F$, making it more prone to such violation and leaving the safe KL region and therefore more vulnerable to training instability in RLVR.

\noindent{\em Empirical analysis.}
Figure~\ref{fig:ortho} confirms this analysis in the principal singular subspace. In the first column, the top panel reports training reward and the bottom panel
reports KL divergence between the rollout policy and the current policy for LoRA, PiSSA, and OLoRA. Compared with PiSSA, OLoRA incurs substantially smaller KL divergence, which by
Theorem~\ref{prop:kl_leash} implies a correspondingl smaller weight update magnitude. This
reduced policy drift partially mitigates the reward collapse observed in PiSSA.

\begin{figure}[t]
      \centering
      \includegraphics[width=1\linewidth]{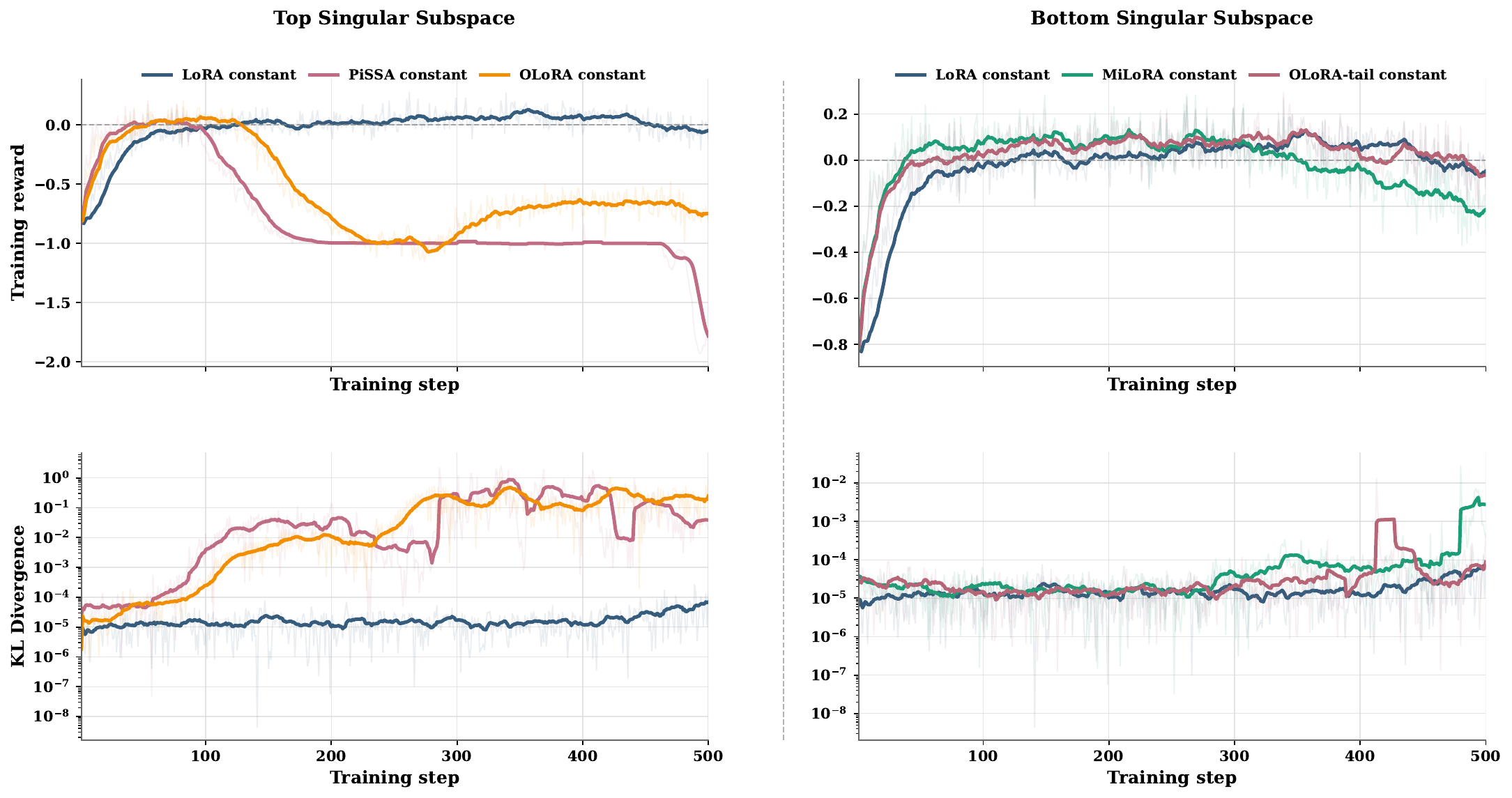} \caption{\textbf{Singular value scaling exacerbates instability beyond subspace selection.}
      We compare constant-learning-rate DAPO training dynamics for LoRA, PiSSA, OLoRA, MiLoRA,
      and OLoRA-tail. The first column compares methods associated with the top singular subspace
      (LoRA, PiSSA, and OLoRA), while the second column compares methods associated with the
      bottom singular subspace (LoRA, MiLoRA, and OLoRA-tail). The first row reports training
      reward, and the second row reports KL divergence between the rollout policy and the current
      policy. Within each subspace, removing singular value scaling, from PiSSA to OLoRA for the
      top subspace and from MiLoRA to OLoRA-tail for the bottom subspace, mitigates reward collapse
      and reduces policy drift. These controlled comparisons show that singular value scaling is
      a distinct source of RLVR instability, beyond the choice of singular subspace.}
      \label{fig:ortho}
  \end{figure}

To verify that this destabilizing effect is not unique to the principal subspace, we extend
the ablation to the minor singular directions via a recent method \textbf{OLoRA-tail} \citep{lab2026scalingpeftmillionpersonal}. As a counterpart to MiLoRA,
  OLoRA-tail targets the same tail singular subspace but removes singular value scaling:
  \begin{equation*}
      B_0 = U_{-r}, \quad A_0 = V_{-r}^\top,
  \end{equation*}
  where \(U_{-r}\) and \(V_{-r}\) denote the last \(r\) columns of \(U\) and \(V\), respectively.
  As shown in the second column of Figure~\ref{fig:ortho}, removing singular value scaling in
  the tail subspace yields markedly more stable dynamics and closely mirrors standard LoRA.

  Taken together, the comparisons in Figure~\ref{fig:ortho} reveal a nuanced interaction
  between subspace geometry and singular value scaling. In the principal singular directions,
  OLoRA delays and weakens the reward collapse observed in PiSSA. In the minor singular
  directions, OLoRA-tail nearly eliminates the instability observed in MiLoRA, yielding
  reward and KL dynamics close to standard LoRA. Thus, although certain singular subspaces
  are more prone to exceeding the KL leash, removing singular value scaling and enforcing
  orthonormality provides a viable path to more stable RLVR training \citep{lab2026scalingpeftmillionpersonal}. This motivates a deeper
  study of LoRA optimization dynamics and a principled initialization strategy that utilizes the geometry of singular subspace while avoiding destabilizing singular value scaling.

\section{Geometry-Preserving Orthonormal Initialization for LoRA}

In this section, we formally analyze the optimization dynamics of LoRA to investigate initialization strategy, studying how initialization can bridge the gap to full fine-tuning while maintaining training stability.  The proof has been deferred to the Appendix \ref{app:proofs_dynamics}. Based on this analysis, we propose orthonormal initialization strategy that preserves the geometry of the pretrained weight space while avoiding singular value scaling.

\subsection{Optimization Dynamics of LoRA}
\label{sec:dynamics}

Since LoRA constrains updates to a rank-$r$ subspace with $r < n$, it often cannot exactly match full fine-tuning. We quantify this approximation gap with respect to the initialization strategy in Theorem~\ref{thm:lora_error}, aiming to reduce the gap as much as possible.

\paragraph{Orthonormal initialization in LoRA leads closer to full fine-tuning.}
Without loss of generality, we consider a general RLVR objective:
\begin{equation}
\mathcal{L}_{\text{RLVR}}(\theta):= 
\mathbb{E}_{x \sim \mathcal{D},\, y \sim \pi_\theta(\cdot|x)}\left[R(x, y)\right] - \beta \cdot \mathrm{KL}(\pi_\theta \| \pi_{\text{ref}}),
    \label{eq:rlvr-general}
\end{equation}
where the reward $R(x, y) \in \{0, 1\}$ equals $1$ if $y$ is a valid solution to $x$ and $0$ otherwise, $\pi_{\text{ref}}$ is some reference policy, and $\beta > 0$ denotes a KL penalty coefficient Following standard LoRA, we let $B_0 = 0$ and consider a more general initialization for $A_0 \in \mathbb{R}^{r \times n}$. Let $T$ denote the total number of training iterations and $t \in \{0, 1, \ldots, T\}$ the current iteration. Consider a input $x$ and a single linear layer with weight matrix $W_t = W_0 + B_t A_t$. The forward pass computes logits $z_t = W_t x$, and the policy $\pi_\theta(y \mid x) = \mathrm{softmax}(z_t)$ gives the probability of generating output $y$, where $\theta$ represent the model parameters, namely $W_t$ at time $t$. At time $t$, the gradients with respect to the two LoRA matrices are $\frac{\partial \mathcal{L}}{\partial A} = B_t^\top G_t, \frac{\partial \mathcal{L}}{\partial B} = G_t A_t^\top$, where $G_t = \nabla_{W_t} \mathcal{L}$ denotes the gradient of the loss with respect to $W_t$.

\begin{assumption}
\label{ass:smooth}
The RLVR loss $\mathcal{L}_{\text{RLVR}}$ from~\eqref{eq:rlvr-general} is $L$-smooth with respect to logits $z = Wx$. The gradient is bounded: $\|G_t\|_F \leq M$ for all $t$.
\end{assumption}
The assumption on the loss holds for many RLVR objectives with binary rewards, with a detailed clarification provided in Appendix~\ref{app:smooth}. The assumption on the gradient is typically enforced via the clipped importance-ratio mechanism, which is widely used in DAPO, GRPO, and other RL algorithms.

\begin{theorem}[LoRA Approximation Error]
\label{thm:lora_error}
Let $W^{\mathrm{full}}_T$ and $W^{\mathrm{LoRA}}_T$ be the weights after $T$ steps of full fine-tuning and LoRA under RLVR. 
Under Assumption~\ref{ass:smooth}, define
$\Gamma_T(\eta)
\coloneqq
\frac{1}{T}\sum_{s=0}^{T-1}
\left(1+L\eta\|x\|_2^2\right)^s$, then for a fixed training horizon $T$ and sufficiently small $\eta$,
\begin{align}\label{eq:lora_error}
\frac{1}{T}\bigl\|W^{\mathrm{LoRA}}_T - W^{\mathrm{full}}_T\bigr\|_F &\le  
M\eta\,\Gamma_T(\eta)\,\bigl\|I_n-A_0^\top A_0\bigr\|_2  +\mathcal{O}(\eta^2).
\end{align}
Furthermore, for $A_0 \in \mathbb{R}^{r \times n}$ with $r < n$, 
$\bigl\| I_n - A_0^\top A_0 \bigr\|_2 \geq 1$
holds with equality if $A_0$ has orthonormal rows, i.e., $A_0 A_0^\top = I_r$.
\end{theorem}
Theorem~\ref{thm:lora_error} shows that when the learning rate $\eta$ is sufficiently small, the approximation error between LoRA and full fine-tuning is controlled by $\Gamma_T(\eta)\|I_n - A_0^\top A_0\|_2$. The factor $\Gamma_T(\eta)$ is independent of the initialization. Thus, orthonormal initialization of $A_0$ minimizes the initialization-dependent term to $1$.

With standard LoRA initialization $B_0 = 0$, Theorem~\ref{thm:lora_error} implies that orthonormal initialization $A_0$ minimizes the approximation gap between LoRA and full fine-tuning in RLVR, making LoRA’s performance closer to that of full fine-tuning.

\paragraph{Orthonormal initialization stabilizes RLVR training.} Beyond minimizing the gap between LoRA and full fine-tuning, orthonormal initialization also controls the update magnitude at each step, which is critical for stable training under RLVR.

\begin{proposition}[Bounded Weight Updates for Orthonormal Initialization]
\label{prop:ortho_bounded}
For the LoRA parameterization $\Delta W^{\mathrm{LoRA}} = BA$ with $B_0=0$ and row-orthonormal $A_0$, the first-step LoRA update satisfies
\begin{equation}
\bigl\|\Delta W^{\mathrm{LoRA}}_1\bigr\|_F
= \eta\,\bigl\|G_0A_0^\top A_0\bigr\|_F
\le \eta\,\|G_0\|_F,
\end{equation}
where $G_0=\nabla_{W_0}\mathcal{L}$.
\end{proposition}
This result shows that, at the first iteration, orthonormal initialization guarantees that the LoRA weight update $\bigl\|\Delta W^{\mathrm{LoRA}}_1\bigr\|_F$ does not exceed that of full fine-tuning in Frobenius norm. This controlled first-step behavior mitigates abrupt policy shifts and supports stable RLVR fine-tuning.

\subsection{Geometry-Preserving Orthonormal Initialization for LoRA}

The theoretical insights above indicate that orthonormal initialization of $A_0$ is beneficial for low-rank fine-tuning in RLVR, both in minimizing the gap to full fine-tuning and in ensuring bounded weight updates. Motivated by this, we propose two initialization schemes that enforce orthonormality of $A_0$ while setting $B_0 = \mathbf{0}_{m \times r}$, illustrated in Figure~\ref{fig:visual} alongside comparisons to prior works. To preserve the geometric structure of the pretrained weight matrix $W_0$, both schemes are derived from its SVD.

\begin{itemize}
    \item \textbf{Principal orthonormal initialization (\RLPO).}
We initialize the adapter $A_0$ using the principal singular vectors:
\vspace{-4pt}
\begin{equation*}
      B_0 = \mathbf{0}_{m \times r}, \quad A_0 = V_r^\top,
\end{equation*}
where $V_r \in \mathbb{R}^{n \times r}$ contains the top-$r$ right singular vectors of $W_0$. This preserves the geometric information of the pretrained model by aligning the adapter with the principal directions of $W_0$. The design is similar in spirit to PiSSA, as both retain the top-$r$ singular directions of $W_0$; however, PiSSA additionally incorporates the singular value scaling.

    \item \textbf{Minor orthonormal initialization (\RLMO).}
Analogously, we define an initialization targeting the minor subspace:
\begin{equation*}
    B_0 = \mathbf{0}_{m \times r}, \quad A_0 = V_{-r}^\top,
\end{equation*}
where $V_{-r} \in \mathbb{R}^{n \times r}$ consists of the bottom-$r$ right singular vectors of $W_0$. While MiLoRA also targets the minor subspace, it incorporates singular-value scaling and a nonzero $B_0$; in contrast, \RLMO adopts an orthonormal $A_0$ with $B_0 = \mathbf{0}$.
\end{itemize}

\section{Experiments and Analysis}

In this section, we conduct experiments to validate our theoretical findings across various benchmarks.
\paragraph{Experimental setup.}
We fine-tune DeepSeek-R1-Distill-Qwen-1.5B using DAPO~\citep{yu2026dapo} on DAPO-Math-17k \citep{yu2026dapo} with rank $r=16$ LoRA applied to all linear layers. We compare standard LoRA, PiSSA, MiLoRA,  \RLPO, and \RLMO across five mathematical reasoning benchmarks: GSM8K~\citep{cobbe2021gsm8k} (1,319 samples), MATH500~\citep{hendrycks2021measuringmathematicalproblemsolving} (500 samples, mean@4), and AIME 2022/2023/2024 (30 samples each, mean@32)~\citep{aime24}. Full experimental details are provided in Appendix~\ref{app:experimental_details}.

\begin{table*}[t]
      \centering
      \renewcommand{\arraystretch}{0.70}
      \begin{tabular*}{\textwidth}{@{\extracolsep{\fill}}lccccc@{}}
          \toprule
          & LoRA & MiLoRA & \RLMO (Ours) & PiSSA & \RLPO (Ours) \\
          \midrule
          $B_0$
          & $\mathbf{0}$
          & $U_{-r}\Sigma_{-r}^{1/2}$
          & $\mathbf{0}$
          & $U_r\Sigma_{r}^{1/2}$
          & $\mathbf{0}$ \\

          $A_0$
          & $\mathcal{N}(0, \frac{1}{n})$
          & $\Sigma_{-r}^{1/2}V_{-r}^\top$
          & $V_{-r}^\top$
          & $\Sigma_{r}^{1/2}V_r^\top$
          & $V_r^\top$ \\
          \midrule
          GSM8K$^{@1}$
          & $75.64_{\pm 1.25}$
          & $75.41_{\pm 1.35}$
          & $\mathbf{76.42}_{\pm 0.76}$
          & $9.65_{\pm 8.20}$
          & $75.59_{\pm 1.14}$ \\

          MATH500$^{@4}$
          & $81.93_{\pm 7.56}$
          & $76.47_{\pm 6.79}$
          & $86.80_{\pm 2.31}$
          & $13.20_{\pm 9.85}$
          & $\mathbf{87.33}_{\pm 1.81}$ \\

          AIME22$^{@32}$
          & $43.33_{\pm 6.67}$
          & $28.89_{\pm 1.92}$
          & $42.22_{\pm 1.92}$
          & $0.00_{\pm 0.00}$
          & $\mathbf{46.67}_{\pm 3.33}$ \\

          AIME23$^{@32}$
          & $38.89_{\pm 5.09}$
          & $30.00_{\pm 3.33}$
          & $41.11_{\pm 5.09}$
          & $0.00_{\pm 0.00}$
          & $\mathbf{42.22}_{\pm 6.94}$ \\

          AIME24$^{@32}$
          & $72.22_{\pm 3.85}$
          & $47.78_{\pm 6.94}$
          & $72.22_{\pm 1.92}$
          & $0.00_{\pm 0.00}$
          & $\mathbf{73.33}_{\pm 0.00}$ \\
          \midrule
          Avg
          & $62.40_{\pm 2.96}$
          & $51.71_{\pm 0.98}$
          & $63.76_{\pm 1.64}$
          & $4.57_{\pm 3.26}$
          & $\mathbf{65.03}_{\pm 0.55}$ \\
          \bottomrule
      \end{tabular*}
      \vspace{4pt}
      \caption{Comparisons of SVD-based LoRA initialization methods with cosine learning rate decay. Results are reported as mean$_{\pm \mathrm{std}}$.}
      \label{tab:svd_comparison}
  \end{table*}

\begin{figure}[t]
    \centering
    \includegraphics[width=0.7\linewidth]{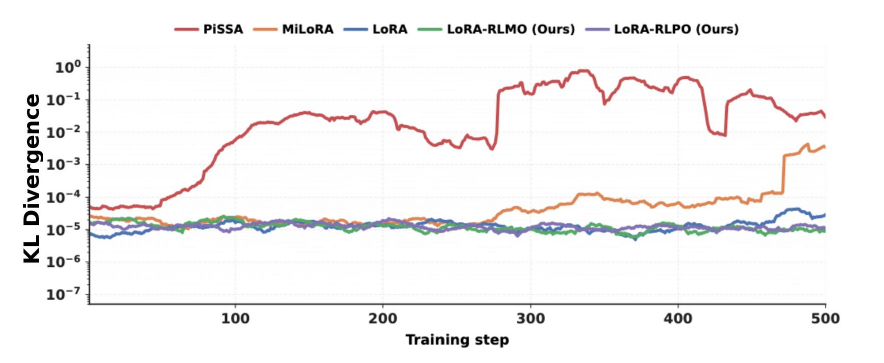}
    \caption{KL divergence during training for different initialization methods. PiSSA shows the highest KL divergence, MiLoRA exhibits intermediate KL growth, and LoRA remains relatively stable. Our proposed methods (\RLMO and \RLPO) maintain the lowest KL trajectories overall, indicating improved training stability.}
    \label{fig:kl_comparison}
\end{figure}

\begin{figure}[t]
    \centering
    \includegraphics[width=0.8\linewidth]{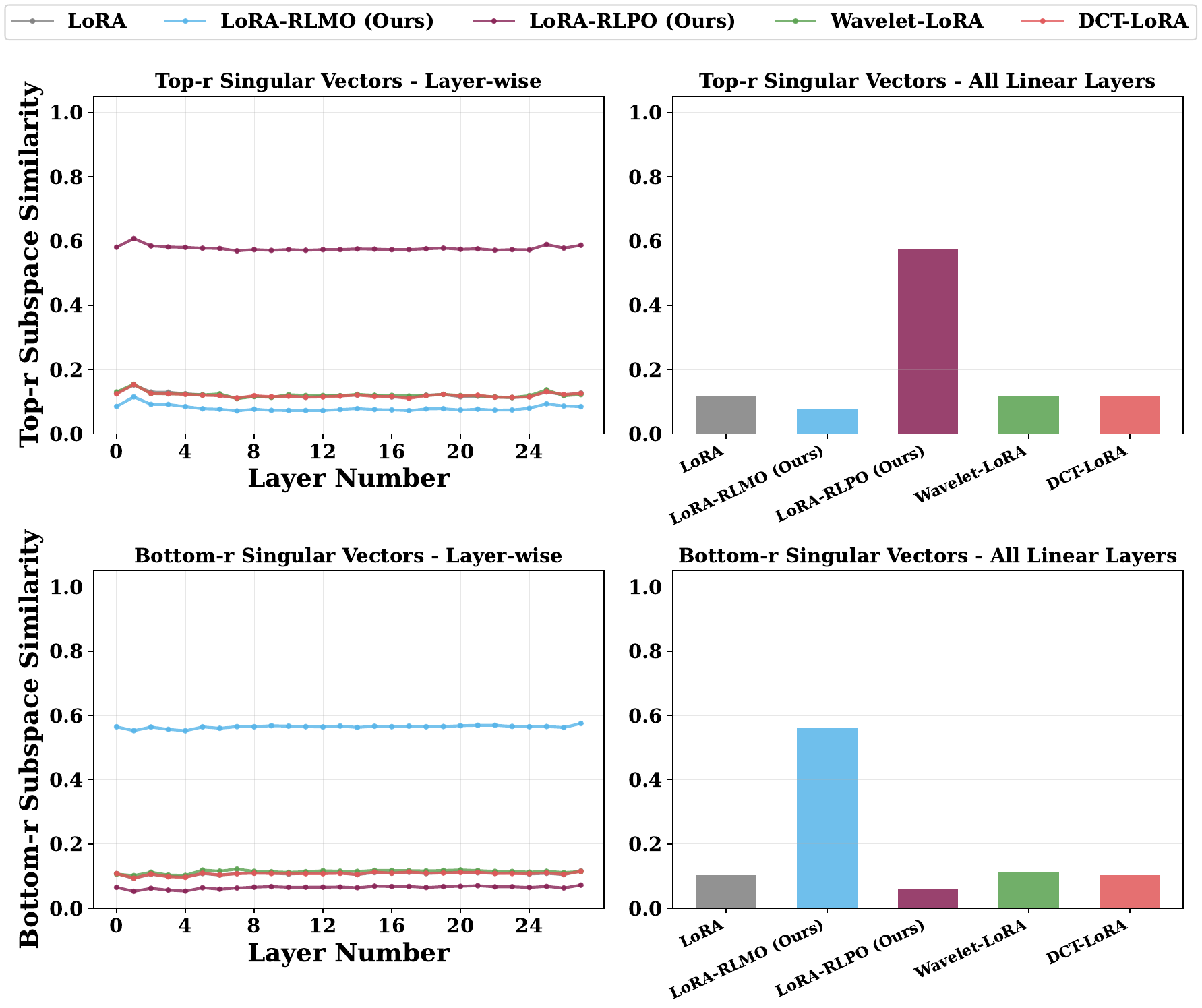}
    \caption{Subspace similarity between learned adapters and singular vectors of pretrained weights $W_0$. Top row: similarity with principal (top-$r$) right singular vectors. Bottom row: similarity with minor (bottom-$r$) right singular vectors. Left column shows per-layer similarity; right columns show averages over all linear layers.}
    \label{fig:subspace_similarity}
\end{figure}

\begin{figure}[t]
    \centering
    \includegraphics[width=1\linewidth]{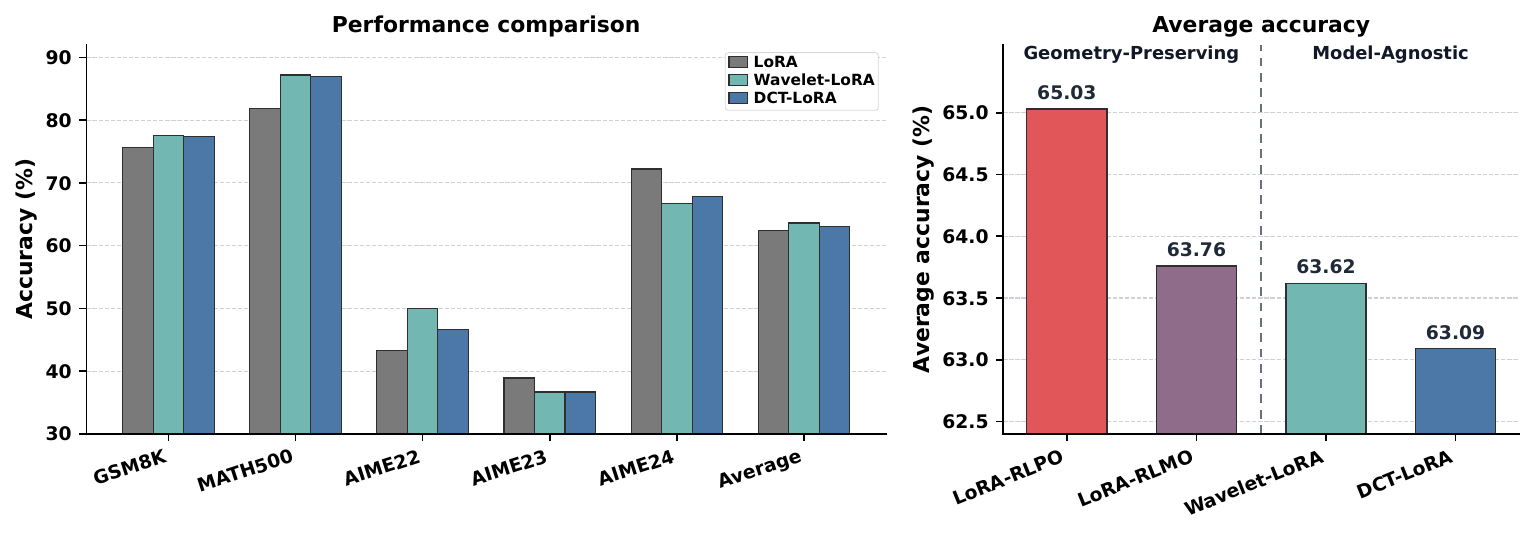}
   \caption{Left: Performance comparison of orthonormal LoRA variants with standard LoRA. Right: Average accuracy comparison between geometry-preserving and model-agnostic methods.}
   \label{fig:orthonormal_comparison}
\end{figure}

\paragraph{Performance.}
Table~\ref{tab:svd_comparison} compares SVD-based LoRA initialization methods at step 500 across five mathematical reasoning
  benchmarks, reporting the mean and standard deviation over multiple seeds. Overall, \RLPO achieves the highest average accuracy
  ($65.03_{\pm 0.55}\%$), followed by \RLMO ($63.76_{\pm 1.64}\%$) and LoRA ($62.40_{\pm 2.96}\%$). \RLPO obtains the best
  performance on MATH500, AIME22 , AIME23, and AIME24
 , while \RLMO achieves the strongest GSM8K result. MiLoRA consistently trails
  behind LoRA on average, and PiSSA performs substantially worse than the other methods in this setting. These results show
  that our geometry-preserving orthonormal initializations improve the stability and effectiveness of SVD-informed LoRA
  variants for RLVR.

\paragraph{Training stability.}
Figure~\ref{fig:kl_comparison} monitors and compares the KL divergence term $  D_{\mathrm{KL}}(\pi_{\theta^+}\,\|\,\pi_\theta)$ in \eqref{eq:kl_constraint} during training for different initialization methods. PiSSA shows the largest KL divergence by a wide margin, while MiLoRA exhibits intermediate KL growth. LoRA, \RLMO, and \RLPO remain in a low-KL regime throughout training. Among them, our proposed methods are particularly stable, with KL trajectories that are consistently comparable to, and even lower than that of standard LoRA. Together with their stronger final evaluation results, these observations indicate that geometry-preserving initialization supports both stable optimization and improved downstream performance in RLVR.

\paragraph{Ablation: orthonormality versus SVD geometry.}
The theoretical analysis implies that an orthonormal $A_0$ minimizes the approximation error to full fine-tuning. To disentangle the two ingredients of the proposed \RLPO and \RLMO: orthonormality and SVD-based geometry information, and determine whether orthonormality alone improves RLVR training or the geometry information also contributes, we introduce two model-agnostic baselines: \textsf{DCT-LoRA} and \textsf{Wavelet-LoRA}. Both of them satisfy the orthonormality condition $A_0 A_0^\top = I_r$ but use no information from the pretrained weight matrix $W_0$. \textsf{DCT-LoRA} initializes $B_0 = 0$ and $A_0 = D_r$, where
$D \in \mathbb{R}^{n \times n}$ is the orthonormal Discrete Cosine
Transform (DCT) matrix~\citep{ahmed2006discrete}. Its entries are
\begin{equation*}
    D_{ij} =
    \alpha_i
    \cos\left(\frac{\pi(2j+1)i}{2n}\right),
    \qquad
    \alpha_i =
    \begin{cases}
    \sqrt{1/n}, & i=0,\\
    \sqrt{2/n}, & i=1,\ldots,n-1,
    \end{cases}
\end{equation*}
for $i,j=0,\ldots,n-1$. Here, $D_r \in \mathbb{R}^{r\times n}$
denotes the first $r$ rows of $D$.  \textsf{Wavelet-LoRA} initializes \(B_0 = 0\) and constructs \(A_0\) as a
  row-orthonormal matrix from a Haar-wavelet-transformed random basis. Let
  \(G \in \mathbb{R}^{r\times n}\) be Gaussian and let \(\mathcal{H}\) denote
  the row-wise Haar wavelet transform. With
  \(Q R = \mathcal{H}(G)^\top\), we set \(A_0=Q^\top\), so that
  \(A_0A_0^\top=I_r\).

We first verify that \RLPO and \RLMO indeed leverage and preserve the geometric structure of the pretrained weights by steering the optimization toward a specific singular subspace, whereas the model-agnostic variants DCT-LoRA and Wavelet-LoRA do not. To this end, we measure the subspace similarity between the learned  model parameter $A$ and the singular vectors of pretrained matrix $W_0$.

For each layer, we compute the similarity as $\|A V\|_F / \|A\|_F$ and report the average across all layers. As shown in Figure~\ref{fig:subspace_similarity}, \RLPO maintains high similarity with the principal singular vectors throughout training, consistent with its initialization from $V_r$, and \RLMO likewise maintains high similarity with the minor singular vectors. In contrast, DCT-LoRA and Wavelet-LoRA exhibit uniformly low similarity across all singular subspaces, confirming that these model-agnostic bases do not exploit the pretrained weight geometry.

Furthermore, Figure~\ref{fig:orthonormal_comparison} (left) shows that both DCT-LoRA and Wavelet-LoRA outperform standard LoRA on most benchmarks, confirming that orthonormality alone improves RLVR training. Figure~\ref{fig:orthonormal_comparison} (right) shows that \RLMO achieves performance comparable to the model-agnostic methods, while \RLPO significantly outperforms all others, indicating that principal-subspace alignment provides substantial benefit beyond orthonormality alone. Thus, both orthonormality and SVD geometry contribute to the strong performance of our methods, with the latter yielding a particularly large gain for \RLPO.

These results also suggest that learning in the principal subspace is not inherently harmful in RLVR, but it is considerably more fragile. Although both \RLPO and PiSSA are initialized in the principal singular subspace, only \RLPO remains stable and achieves the best overall performance (Table~\ref{tab:svd_comparison}). Subspace choice alone therefore does not determine the outcome; rather, our results indicate that geometry-informed orthonormal initialization combined with cosine learning-rate decay is sufficient to make principal-subspace learning both stable and effective in RLVR.

\paragraph{Further ablation study.}
We provide additional ablation study regarding the proposed initialization strategy in Appendix~\ref{app:additional_experiments}. First, regarding the recent finding in~\citet{yin2025evaluatingparameterefficientmethods}, we revisits the failure modes of PiSSA and MiLoRA in Appendix~\ref{app:pissa_analysis}: we show that MiLoRA's initial update is not negligible in norm (see Table~\ref{tab:init_norms}), and that the tail singular spectrum remains nonzero (see Figure~\ref{fig:singular_values}), suggesting that its instability is not induced by near-zero tail initialization, differ from the finding in \citet{yin2025evaluatingparameterefficientmethods} that may be of independent interest to readers. Second, we evaluates the generalization of the proposed geometry-preserving initializations methods beyond the main 1.5B model on mathematical reasoning tasks in Appendix~\ref{app:svd_cost}, including additional task domain code-generation, addtional model families Llama 3.2-3B-Instruct model (Table~\ref{tab:code_gen_llama}), larger-size Qwen2.5-7B-Instruct model (Table~\ref{tab:qwen_7b_results}), learning-rate sensitivity (Figure~\ref{fig:lr_sweep}), and one-time SVD preprocessing cost (Table~\ref{tab:svd_cost}). Finally, we demonstrate that the proposed methods also improve beyond RL fine-tuning frameworks to supervised fine-tuning in Appendix~\ref{sec:sft_generalization}, where we evaluate on  GLUE and GSM8K and show with stronger final performance and faster convergence (Table~\ref{tab:sft_results} and Figure~\ref{fig:sft_train_loss}).

\section{Conclusion}

In this work, we studied why geometry-informed LoRA variants that are effective in supervised fine-tuning can become unstable under RLVR. We identify two primary factors governing this instability: (1) \emph{subspace geometry}, which fundamentally shapes the optimization trajectory and concentrates the energy of parameter updates in certain subspaces; and (2) \emph{singular-value scaling}, a distinct destabilizing factor that amplifies gradient magnitudes and drives rapid violations of the KL-divergence constraint. We provide theoretical results showing that orthonormal initialization, paired with the standard LoRA choice of zero-initializing $B_0$, minimizes the approximation gap to full fine-tuning and helps control update magnitudes. Building on these insights, we propose \RLPO and \RLMO, two geometry-preserving orthonormal initialization schemes that retain useful spectral information from the pretrained weights without singular-value scaling. Empirical results on mathematical reasoning benchmarks show that these methods stabilize RLVR training and improve downstream performance. Overall, our findings suggest that orthonormal, geometry-aware initialization offers a principled and effective foundation for low-rank adaptation in RLVR.

\section*{Acknowledgments}
Laixi Shi acknowledges funding support from MERL. Ruijia Zhang thanks Chenliang Li, Di Zhang, Wenbin Wang, Qihan Liu, Pony Ma, Andrew Chen, Qingyu Yin, and the anonymous reviewers for their insightful discussions and constructive feedback, which helped improve this paper. Ruijia Zhang also thanks Mind Lab for broader and larger-scale experimental validation of the theoretical insights in this work.

\bibliography{References}
\bibliographystyle{apalike} 

\appendix
\onecolumn

\section{Proof of LoRA Optimization Dynamics}
\label{app:proofs_dynamics}

In this section, we provide detailed proofs for the results in Section~\ref{sec:dynamics}. We begin with some preliminary results that will be used in the subsequent proofs.

\paragraph{Proof of Assumption~\ref{ass:smooth}: smoothness of the loss function.}
\label{app:smooth}
We verify that Assumption~\ref{ass:smooth} holds for two widely used fine-tuning formulations: supervised fine-tuning (SFT) with the cross-entropy loss, and reinforcement learning with verifiable rewards (RLVR) with the policy-gradient surrogate loss. Consider a linear layer $z = Wx$, and let $p(z) := \pi_\theta(\cdot \mymid x) = \operatorname{softmax}(z) \in \mathbb{R}^{d}$ denote the token-probability vector over a vocabulary of size $d$, with $i$-th entry $p_i(z) = \exp(z_i)\big/\sum_{j=1}^d \exp(z_j)$.

\emph{SFT with cross-entropy loss.}
For a one-hot target token $a$, the cross-entropy loss is
\begin{equation}\label{eq:ce-loss}
    \mathcal{L}_{\mathrm{CE}}(z) \;=\; -\sum_{i=1}^{d} \mathbbm{1}\{i = a\}\log p_i(z) \;=\; -\log p_a(z),
\end{equation}
where $\mathbbm{1}\{\cdot\}$ is the indicator function. Its gradient and Hessian with respect to the logits are
\begin{equation*}
    \nabla_z \mathcal{L}_{\mathrm{CE}}(z) \;=\; p(z) - e_a, \qquad
    \nabla_z^2 \mathcal{L}_{\mathrm{CE}}(z) \;=\; H(z) \;:=\; \operatorname{diag}(p(z)) - p(z)\,p(z)^\top,
\end{equation*}
where $e_a$ is the standard basis vector corresponding to token $a$. Since $H(z)$ is a symmetric  positive semidefinite matrix, for any $v \in \mathbb{R}^d$ we have
\begin{equation}\label{eq:ce-hessian-bound}
    0 \;\leq\; v^\top H(z)\, v
    \;=\;
    \sum_{i=1}^d p_i(z)\,v_i^2
    -
    \Bigl(\sum_{i=1}^d p_i(z)\,v_i\Bigr)^{\!2}
    \;\leq\;
    \sum_{i=1}^d p_i(z)\,v_i^2
    \;\leq\;
    \|v\|_2^2.
\end{equation}
Hence $\|H(z)\|_{\mathrm{op}} = \sup_{\|v\|_2 = 1} v^\top H(z)\, v \leq 1$, so $\mathcal{L}_{\mathrm{CE}}$ is $1$-smooth with respect to the logits $z$, i.e., $L = 1$.

\emph{RLVR with policy gradient.}
We consider a variance-reduced policy-gradient objective, a standard surrogate of~\eqref{eq:rlvr} that is widely used in RL:
\begin{equation}\label{eq:pg-loss2}
    \mathcal{L}_{\mathrm{PG}}(\theta)
    \;=\; \mathbb{E}_{x \sim \mathcal{D},\, y \sim \pi_\theta(\cdot \mymid x)}
          \!\bigl[\widehat{A}(x, y)\log \pi_\theta(y \mymid x)\bigr]
    \;=\; \mathbb{E}_{x \sim \mathcal{D},\, y \sim \pi_\theta(\cdot \mymid x)}
          \!\bigl[\widehat{A}(x, y)\log p_y(z)\bigr],
\end{equation}
where $\widehat{A}$ is an estimate of the advantage function. As in standard policy-gradient updates, both $\widehat{A}$ and the sampling distribution are treated as fixed when differentiating with respect to the current logits, \eqref{eq:pg-loss2} thus takes the same form as~\eqref{eq:ce-loss}, but reweighted by $\widehat{A}$. Consequently, the Hessian with respect to $z$ is
\begin{equation*}
    \nabla_z^2 \mathcal{L}_{\mathrm{PG}}
    \;=\; -\,\mathbb{E}_{x \sim \mathcal{D},\, y \sim \pi_\theta(\cdot \mymid x)}
          \!\bigl[\widehat{A}(x, y)\bigr]
          \bigl[\operatorname{diag}(p(z)) - p(z)\,p(z)^\top\bigr],
\end{equation*}
and invoking~\eqref{eq:ce-hessian-bound} yields
\begin{equation}\label{eq:pg-hessian-bound}
    \bigl\|\nabla_z^2 \mathcal{L}_{\mathrm{PG}}\bigr\|_{\mathrm{op}}
    \;\leq\;
    \mathbb{E}_{x \sim \mathcal{D},\, y \sim \pi_\theta(\cdot \mymid x)}
    \!\bigl[\,|\widehat{A}(x, y)|\,\bigr].
\end{equation}
For binary verifiable rewards $R \in \{0, 1\}$, we have $|\widehat{A}(x, y)| \leq 1$, giving $L = 1$ in this case as well.

\paragraph{First-step LoRA update.}
\label{app:lemma_first_step}
We first introduce the following lemma about one-step LoRA update that will be used to compare the LoRA trajectory with the full fine-tuning trajectory.

\begin{lemma}[First-step LoRA update]
\label{lem:first_step}
Let $W_0 \in \mathbb{R}^{m \times n}$ be the pretrained weight and consider the LoRA parameterization $W = W_0 + BA$ with $B_0 = \mathbf{0}_{m \times r}$ and $A_0 \in \mathbb{R}^{r \times n}$. Under one gradient-descent step with step size $\eta$, the induced LoRA update satisfies
\begin{equation}
\Delta W_1^{\mathrm{LoRA}} = \Delta W_1^{\mathrm{full}} A_0^\top A_0,
\end{equation}
where $\Delta W_1^{\mathrm{full}} = -\eta \left.\nabla_W \mathcal{L}(W)\right|_{W=W_0}$ and $\Delta W_1^{\mathrm{LoRA}} = B_1A_1$.
\end{lemma}

\begin{proof}
Define the full-weight gradient at initialization as
\[
G_0 \coloneqq \left.\nabla_W \mathcal{L}(W)\right|_{W=W_0} \in \mathbb{R}^{m \times n}.
\]
By the chain rule applied to $W=W_0+BA$, the gradients of the LoRA factors at $(B,A)=(B_0,A_0)$ are
\begin{align}\label{eq:lora-gradient}
\left.\nabla_A \mathcal{L}(W_0+BA)\right|_{(B,A)=(B_0,A_0)}
= B_0^\top G_0 = \mathbf{0},
\qquad
\left.\nabla_B \mathcal{L}(W_0+BA)\right|_{(B,A)=(B_0,A_0)}
= G_0 A_0^\top .
\end{align}
Therefore,
\[
A_1 = A_0, \qquad
B_1 = -\eta G_0 A_0^\top .
\]
The first LoRA weight update is consequently
\[
\Delta W_1^{\mathrm{LoRA}}
= B_1A_1
= -\eta G_0A_0^\top A_0.
\]
Since $\Delta W_1^{\mathrm{full}}=-\eta G_0$, the claimed identity follows.
\end{proof}

\subsection{Proof of Theorem~\ref{thm:lora_error} and Proposition~\ref{prop:ortho_bounded}}
\label{app:lora_error}

We prove the main approximation bound in Theorem~\ref{thm:lora_error} first, then establish the additional lower-bound statement in the theorem, and finally prove Proposition~\ref{prop:ortho_bounded}.

\paragraph{Step 1: introducing useful notations.}
Let $W_t^{\mathrm{LoRA}}$ and $W_t^{\mathrm{full}}$ denote the LoRA and fulfl fine-tuning weights after $t$ gradient-descent steps, initialized from the same pretrained weight:
\[
W_0^{\mathrm{LoRA}} = W_0^{\mathrm{full}} = W_0.
\]
For LoRA, we express the weights as
\[
W_t^{\mathrm{LoRA}} = W_0 + B_tA_t,
\qquad
B_0 = \mathbf{0}, \quad A_0 \in \mathbb{R}^{r \times n}.
\]
Define the approximation error at step $t$ between LoRA and full fine-tuning as
\[
E_t \coloneqq \left\|W_t^{\mathrm{LoRA}}-W_t^{\mathrm{full}}\right\|_F .
\]
We will prove a recursion for $E_t$ and then optimize the initialization-dependent factor $\|I_n-A_0^\top A_0\|_2$.

Fixing an input $x$, we denote the logits as
\[
z_t^{\mathrm{LoRA}} \coloneqq W_t^{\mathrm{LoRA}}x,
\qquad
z_t^{\mathrm{full}} \coloneqq W_t^{\mathrm{full}}x .
\]
Throughout this section, we use the following notations for brevity:
\[
G_t^{\mathrm{LoRA}}
\coloneqq
\left.\nabla_W \mathcal{L}(W)\right|_{W=W_t^{\mathrm{LoRA}}}
= g_t^{\mathrm{LoRA}}x^\top,
\qquad
G_t^{\mathrm{full}}
\coloneqq
\left.\nabla_W \mathcal{L}(W)\right|_{W=W_t^{\mathrm{full}}}
= g_t^{\mathrm{full}}x^\top,
\]
where $g_t^{\mathrm{LoRA}}$ and $g_t^{\mathrm{full}}$ are the corresponding gradients with respect to logits.

\paragraph{Step 2: Proof of the main results in \eqref{eq:lora_error}.}
We next compare the LoRA and full fine-tuning dynamics step by step and then unroll the resulting recursion.

The full fine-tuning update is
\[
W_{t+1}^{\mathrm{full}}
= W_t^{\mathrm{full}}-\eta G_t^{\mathrm{full}}.
\]
For LoRA, gradient descent on the factors gives
\[
A_{t+1}=A_t-\eta B_t^\top G_t^{\mathrm{LoRA}},
\qquad
B_{t+1}=B_t-\eta G_t^{\mathrm{LoRA}}A_t^\top .
\]
With the above results in hand, combined with the fact that $B_0=\mathbf{0}$, we have for any fixed horizon $T$ and sufficiently small $\eta$,
\[
B_t=\mathcal{O}(\eta),
\qquad
A_t=A_0+\mathcal{O}(\eta^2),
\qquad 0\le t\le T.
\]
Thus the LoRA weight increment satisfies
\begin{align}
W_{t+1}^{\mathrm{LoRA}}-W_t^{\mathrm{LoRA}}
&= B_{t+1}A_{t+1}-B_tA_t \notag \\
&= -\eta G_t^{\mathrm{LoRA}}A_0^\top A_0+\mathcal{O}(\eta^3).
\label{eq:lora_step_expansion}
\end{align}
Comparing the two dynamics, we obtain
\begin{align}
W_{t+1}^{\mathrm{LoRA}}-W_{t+1}^{\mathrm{full}}
&= W_t^{\mathrm{LoRA}}-W_t^{\mathrm{full}}
-\eta G_t^{\mathrm{LoRA}}(A_0^\top A_0-I_n)
+\eta\left(G_t^{\mathrm{full}}-G_t^{\mathrm{LoRA}}\right)
+\mathcal{O}(\eta^3).
\end{align}
Taking Frobenius norms and applying the triangle inequality gives
\begin{align}
E_{t+1}
&\le
E_t
+\eta\left\|G_t^{\mathrm{LoRA}}(I_n-A_0^\top A_0)\right\|_F
+\eta\left\|G_t^{\mathrm{LoRA}}-G_t^{\mathrm{full}}\right\|_F
+\mathcal{O}(\eta^3) \notag \\
&\le E_t
+\eta\left\|G_t^{\mathrm{LoRA}}\right\|_F \left\|I_n-A_0^\top A_0\right\|_2
+\eta\left\|G_t^{\mathrm{LoRA}}-G_t^{\mathrm{full}}\right\|_F
+\mathcal{O}(\eta^3) \notag \\
&\le
E_t
+M\eta\left\|I_n-A_0^\top A_0\right\|_2
+\eta\left\|G_t^{\mathrm{LoRA}}-G_t^{\mathrm{full}}\right\|_F
+\mathcal{O}(\eta^3),
\label{eq:lora_error_recursion_pre}
\end{align}
where the second inequality uses $\|G_t^{\mathrm{LoRA}}\|_F \leq M$ in Assumption~\ref{ass:smooth}. Then it remains to bound the gradient discrepancy. Invoking
\[
G_t^{\mathrm{LoRA}}-G_t^{\mathrm{full}}
=\left(g_t^{\mathrm{LoRA}}-g_t^{\mathrm{full}}\right)x^\top,
\]
we have
\[
\left\|G_t^{\mathrm{LoRA}}-G_t^{\mathrm{full}}\right\|_F
=
\left\|g_t^{\mathrm{LoRA}}-g_t^{\mathrm{full}}\right\|_2\|x\|_2 .
\]
Applying the $L$-smoothness of the loss with respect to logits in Assumption~\ref{ass:smooth}, one has
\begin{align}
\left\|g_t^{\mathrm{LoRA}}-g_t^{\mathrm{full}}\right\|_2
&\le
L\left\|z_t^{\mathrm{LoRA}}-z_t^{\mathrm{full}}\right\|_2 =
L\left\|(W_t^{\mathrm{LoRA}}-W_t^{\mathrm{full}})x\right\|_2 \notag \le
L E_t\|x\|_2 .
\label{eq:logit_smooth_bound}
\end{align}
Substituting the above results into~\eqref{eq:lora_error_recursion_pre} yields
\begin{equation}
\label{eq:lora_error_recursion}
E_{t+1}
\le
\left(1+L\eta\|x\|_2^2\right)E_t
+M\eta\left\|I_n-A_0^\top A_0\right\|_2
+\mathcal{O}(\eta^3).
\end{equation}
To continue, as $E_0=\|W_0 - W_0\|_F =0$, recursively applying \eqref{eq:lora_error_recursion} gives
\[
E_T
\le
\left(M\eta\left\|I_n-A_0^\top A_0\right\|_2+\mathcal{O}(\eta^3)\right)
\sum_{s=0}^{T-1}(1+L\eta\|x\|_2^2)^s .
\]
Dividing by $T$ and using
\[
\Gamma_T(\eta)
\coloneqq
\frac{1}{T}\sum_{s=0}^{T-1}
\left(1+L\eta\|x\|_2^2\right)^s
\]
gives
\[
\frac{E_T}{T} = \frac{1}{T}
\left\|W_T^{\mathrm{LoRA}}-W_T^{\mathrm{full}}\right\|_F
\le
M\eta\,\Gamma_T(\eta)\,
\left\|I_n-A_0^\top A_0\right\|_2
+\mathcal{O}(\eta^2).
\]

\paragraph{Step 3: lower bound of $\left\|I_n-A_0^\top A_0\right\|_2$.}
It remains to justify the lower bound on the initialization-dependent factor and the case of equality for row-orthonormal initialization. Since rank $r<n$,
\[
\operatorname{rank}(A_0^\top A_0)\le \operatorname{rank}(A_0)\le r<n.
\]
Therefore, $A_0^\top A_0$ has at least $n-r$ zero eigenvalues, so $I_n-A_0^\top A_0$ has at least $n-r$ eigenvalues equal to $1$. Therefore, as $I_n-A_0^\top A_0$ is a symmetric matrix, its spectral norm is equal to the largest absolute value of its eigenvalues, which is at least $1$, namely,
\begin{equation}
\label{eq:lowerbd_rankdef}
\left\|I_n-A_0^\top A_0\right\|_2 \ge 1.
\end{equation}
If $A_0$ has orthonormal rows, then $A_0A_0^\top=I_r$ and $A_0^\top A_0$ is the orthogonal projector onto $\operatorname{row}(A_0)$. Hence $I_n-A_0^\top A_0$ is the orthogonal projector onto $\operatorname{row}(A_0)^\perp$, and
\[
\left\|I_n-A_0^\top A_0\right\|_2
=1.
\]
Combining this equality with~\eqref{eq:lowerbd_rankdef} shows that row-orthonormal initialization attains the minimum possible value of the initialization-dependent factor $\left\|I_n-A_0^\top A_0\right\|_2$ for the gap between the LoRA and full-fined models.

\paragraph{Step 4: Proof of Proposition~\ref{prop:ortho_bounded}.}
\label{app:ortho_bounded}

The proposition follows by applying Lemma~\ref{lem:first_step} and using the fact that $A_0^\top A_0$ is a projection under row-orthonormal initialization.

Define
\[
G_0 \coloneqq \left.\nabla_W \mathcal{L}(W)\right|_{W=W_0}.
\]
If $A_0$ has orthonormal rows, then $A_0^\top A_0$ is an orthogonal projection and $\|A_0^\top A_0\|_2=1$. By Lemma~\ref{lem:first_step},
\[
\Delta W_1^{\mathrm{LoRA}}=-\eta G_0A_0^\top A_0.
\]
Using the mixed Frobenius--spectral norm inequality,
\[
\|AB\|_F \le \|A\|_F\|B\|_2,
\]
we obtain
\[
\left\|\Delta W_1^{\mathrm{LoRA}}\right\|_F
=\eta\|G_0A_0^\top A_0\|_F
\le
\eta\|G_0\|_F\|A_0^\top A_0\|_2
=
\eta\|G_0\|_F.
\]

\section{Proof of Gradient Amplification of PiSSA over OLoRA}
\label{app:proof_pissa_olora_amplification}
In this section, we prove Theorem~\ref{thm:pissa_amplification}, which explains why PiSSA is more aggressive than OLoRA even when both are initialized on the same principal singular subspace. 
We prove the first-order comparison between PiSSA and OLoRA under the residual parameterization
\[
W=(W_0-B_0A_0)+BA,
\]
which ensures that both methods start from the same effective weight $W_0$. For both PiSSA and OLoRA, for a fixed input $x$, define
\[
G_0 \coloneqq \left.\nabla_W\mathcal{L}(W)\right|_{W=W_0}=gx^\top.
\]
Then we can represent $G_0$ by expanding $g$ and $x$ in the singular-vector bases of $W_0$. Let $W_0=U\Sigma V^\top$, with the left singular vectors $\{u_i\}$ and right singular vectors $\{v_j\}$ form orthonormal bases. Denoting
\begin{align}
\alpha_i\coloneqq u_i^\top g,\qquad
\beta_j\coloneqq v_j^\top x,
\end{align}
we can express $g$ and $x$ in the SVD basis as
\begin{align}
g=\sum_i (u_i^\top g)u_i=\sum_i \alpha_i u_i,
\qquad
x=\sum_j (v_j^\top x)v_j=\sum_j \beta_j v_j,
\end{align}
Therefore,
\begin{align}
G_0=gx^\top = \Big(\sum_i \alpha_i u_i\Big)
\Big(\sum_j \beta_j v_j\Big)^\top=
\sum_{i,j}\alpha_i\beta_j u_i v_j^\top,
\end{align}
where matrices $\{u_i v_j^\top\}_{i,j}$ are orthonormal under the Frobenius inner product.
To continue, recalling the gradient for the first-step in \eqref{eq:lora-gradient} gives
\[
B_1=B_0-\eta G_0A_0^\top,
\qquad
A_1=A_0-\eta B_0^\top G_0,
\]
leading to the update
\begin{align}
\Delta W_1=B_1A_1-B_0A_0 
&=-\eta\left(B_0B_0^\top G_0+G_0A_0^\top A_0\right)
  +\eta^2G_0A_0^\top B_0^\top G_0.
\label{eq:svd_init_first_step}
\end{align}
Thus, up to second-order terms in $\eta$, the update is governed by the two projection factors $B_0B_0^\top$ and $A_0^\top A_0$. 

Before proceeding to the main proofs, we introduce two lemmas: the first gives the one-step updates for PiSSA and LoRA, while the second compares their coefficients.

\paragraph{First-step updates for PiSSA and OLoRA.}
Let $\mathcal{R}=\{1,\ldots,r\}$ denote the retained principal components. 
\begin{lemma}[PiSSA and OLoRA first-step updates]
\label{lem:pissa_olora_expansions}
For PiSSA with $B_0=U_r\Sigma_r^{1/2}$ and $A_0=\Sigma_r^{1/2}V_r^\top$, and for OLoRA with $B_0=U_r$ and $A_0=V_r^\top$, the first-step updates satisfy
\[
\Delta W_1^{s}
=-\eta\sum_{i,j}c_{ij}^{s}\alpha_i\beta_j\,u_i v_j^\top
+\mathcal{O}(\eta^2),
\qquad
s\in\{\mathrm{PiSSA},\mathrm{OLoRA}\},
\]
where
\[
c_{ij}^{\mathrm{PiSSA}}
=\sigma_i\mathbf{1}_{\{i\in\mathcal{R}\}}
+\sigma_j\mathbf{1}_{\{j\in\mathcal{R}\}},
\qquad
c_{ij}^{\mathrm{OLoRA}}
=\mathbf{1}_{\{i\in\mathcal{R}\}}
+\mathbf{1}_{\{j\in\mathcal{R}\}}.
\]
\end{lemma}

\begin{proof}
For PiSSA, applying $B_0=U_r\Sigma_r^{1/2}$ and $A_0=\Sigma_r^{1/2}V_r^\top$ gives
\[
B_0B_0^\top=U_r\Sigma_rU_r^\top,
\qquad
A_0^\top A_0=V_r\Sigma_rV_r^\top,
\]
leading to the two factors
\[
U_r\Sigma_rU_r^\top G_0
=\sum_{i,j}\sigma_i\mathbf{1}_{\{i\in\mathcal{R}\}}\alpha_i\beta_j\,u_i v_j^\top, \qquad 
G_0V_r\Sigma_rV_r^\top
=\sum_{i,j}\sigma_j\mathbf{1}_{\{j\in\mathcal{R}\}}\alpha_i\beta_j\,u_i v_j^\top.
\]
Substitution into~\eqref{eq:svd_init_first_step} yields
\[
\Delta W_1^{\mathrm{PiSSA}}
=-\eta\sum_{i,j}
\left(
\sigma_i\mathbf{1}_{\{i\in\mathcal{R}\}}
+\sigma_j\mathbf{1}_{\{j\in\mathcal{R}\}}
\right)
\alpha_i\beta_j\,u_i v_j^\top
+\mathcal{O}(\eta^2).
\]
Analogously, for OLoRA,
\[
B_0B_0^\top=U_rU_r^\top,
\qquad
A_0^\top A_0=V_rV_r^\top,
\]
leading to 
\[
U_rU_r^\top G_0
=\sum_{i,j}\mathbf{1}_{\{i\in\mathcal{R}\}}\alpha_i\beta_j\,u_i v_j^\top, \qquad
G_0V_rV_r^\top
=\sum_{i,j}\mathbf{1}_{\{j\in\mathcal{R}\}}\alpha_i\beta_j\,u_i v_j^\top.
\]
Substitution into~\eqref{eq:svd_init_first_step} yields
\[
\Delta W_1^{\mathrm{OLoRA}}
=-\eta\sum_{i,j}
\left(
\mathbf{1}_{\{i\in\mathcal{R}\}}
+\mathbf{1}_{\{j\in\mathcal{R}\}}
\right)
\alpha_i\beta_j\,u_i v_j^\top
+\mathcal{O}(\eta^2).
\]
\end{proof}

\begin{lemma}[Coefficient comparison]
\label{lem:pissa_olora_coeff}
For all $(i,j)$,
\[
c_{ij}^{\mathrm{PiSSA}}\ge \sigma_r c_{ij}^{\mathrm{OLoRA}}.
\]
\end{lemma}

\begin{proof}
If $i\in\mathcal{R}$, then $\sigma_i\ge\sigma_r$; if $j\in\mathcal{R}$, then $\sigma_j\ge\sigma_r$. Therefore,
\begin{align*}
c_{ij}^{\mathrm{PiSSA}}
&=\sigma_i\mathbf{1}_{\{i\in\mathcal{R}\}}
+\sigma_j\mathbf{1}_{\{j\in\mathcal{R}\}} \ge
\sigma_r\left(
\mathbf{1}_{\{i\in\mathcal{R}\}}
+\mathbf{1}_{\{j\in\mathcal{R}\}}
\right)
=\sigma_r c_{ij}^{\mathrm{OLoRA}}.
\end{align*}
\end{proof}
Now we are positioned to prove the main result.

\paragraph{Proof of Theorem~\ref{thm:pissa_amplification}.} ApplyingLemma~\ref{lem:pissa_olora_expansions} gives
\[
\Delta W_1^{\mathrm{PiSSA}}
=-\eta\sum_{i,j}c_{ij}^{\mathrm{PiSSA}}\alpha_i\beta_j\,u_i v_j^\top
+\mathcal{O}(\eta^2),
\qquad
\Delta W_1^{\mathrm{OLoRA}}
=-\eta\sum_{i,j}c_{ij}^{\mathrm{OLoRA}}\alpha_i\beta_j\,u_i v_j^\top
+\mathcal{O}(\eta^2).
\]
As the basis $\{u_i v_j^\top\}_{i,j}$ is Frobenius-orthonormal, one has
\[
\Big\|
\sum_{i,j}c_{ij}^{\mathrm{PiSSA}}\alpha_i\beta_j\,u_i v_j^\top
\Big\|_F^2
=\sum_{i,j}\left(c_{ij}^{\mathrm{PiSSA}}\right)^2\alpha_i^2\beta_j^2,
\qquad
\Big\|
\sum_{i,j}c_{ij}^{\mathrm{OLoRA}}\alpha_i\beta_j\,u_i v_j^\top
\Big\|_F^2
=\sum_{i,j}\left(c_{ij}^{\mathrm{OLoRA}}\right)^2\alpha_i^2\beta_j^2.
\]
Applying Lemma~\ref{lem:pissa_olora_coeff} to all $(i,j)$ gives
\[
\Big\|
\sum_{i,j}c_{ij}^{\mathrm{PiSSA}}\alpha_i\beta_j\,u_i v_j^\top
\Big\|_F
\ge
\sigma_r
\Big\|
\sum_{i,j}c_{ij}^{\mathrm{OLoRA}}\alpha_i\beta_j\,u_i v_j^\top
\Big\|_F,
\]
and then
\[
\left\|\Delta W_1^{\mathrm{PiSSA}}\right\|_F
\ge
\sigma_r\left\|\Delta W_1^{\mathrm{OLoRA}}\right\|_F
+\mathcal{O}(\eta^2).
\]
Furthermore, whenever the leading OLoRA update is nonzero, we have
\[
\liminf_{\eta\to 0}
\frac{\left\|\Delta W_1^{\mathrm{PiSSA}}\right\|_F}
{\left\|\Delta W_1^{\mathrm{OLoRA}}\right\|_F}
\ge \sigma_r.
\]

 \section{Experimental Details}
  \label{app:experimental_details}

\subsection{Training details of RLVR}

  \paragraph{Model and dataset for training.}
We conduct our main DAPO experiments on DeepSeek-R1-Distill-Qwen-1.5B. All methods are trained on the DAPO-Math-17k dataset, which contains 17{,}000 mathematical reasoning problems with verifiable answers. We apply LoRA adapters to all linear layers of the policy model, including the attention and MLP projections. Unless otherwise specified, all methods share the same base model, training data, reward function, and optimization setup.

\paragraph{Implementation setup.}
Training prompts are taken from the \texttt{prompt} field of the DAPO-Math-17k parquet file. We apply left truncation with a maximum prompt length of 512 tokens. During rollout generation, the maximum response length is set to 16{,}384 tokens,  matching the long-reasoning regime used by DAPO. For each prompt, we sample 8 candidate responses by drawing directly from the model's
output distribution without any modification. To clarify the sampling hyperparameters, temperature $\tau$ controls the randomness of the generation distribution. Top-$p$ (nucleus sampling) restricts sampling to the smallest token set whose cumulative probability exceeds $p$, while top-$k$ limits the pool to the $k$ most probable tokens (with $k = -1$ indicating this filter is disabled). Specifically, here we use $\tau=1.0$, top-$p=1.0$, top-$k=-1$, i.e., no temperature scaling and no vocabulary truncation.

  All experiments are implemented in the verl \citep{sheng2025hybridflow} framework. We use vLLM \citep{kwon2023efficient} for rollout generation and Flash Attention \citep{dao2024flashattention} for efficient attention computation.
  Training is run on 8 NVIDIA A100-SXM4-80GB GPUs with FSDP \citep{zhao2023pytorch}. The rollout engine uses tensor parallelism of size 2 during training \citep{shoeybi2019megatron}. We enable
  gradient checkpointing, remove-padding optimization, chunked prefill, dynamic batch sizing, actor parameter offload, actor optimizer offload, and
  reference parameter offload \citep{chen2016training,ren2021zero}. 

\paragraph{Hyperparameter settings.}
We use two training configurations for our DAPO experiments. The main configuration uses a constant learning-rate schedule and serves as the primary setting for comparing LoRA, PiSSA, and MiLoRA. We additionally run cosine-decay variants to study the effect of the learning-rate schedule, keeping all remaining hyperparameters matched to the corresponding constant learning-rate runs. Unless otherwise stated, all methods use the DAPO training objective with group-relative advantage estimation, LoRA-style adapters on all linear layers, 8 responses per prompt, and the same DAPO-style objective without an explicit KL reward penalty or actor KL loss.

In the constant learning-rate setting, we train for 500 optimization steps using AdamW with no warmup, weight decay $0.1$, and gradient clipping at $1.0$. For the standard LoRA baseline and the rank-16 variants, we use learning rate $1\times 10^{-5}$, rank $r=16$, and scaling parameter $\alpha=32$, corresponding to an effective scaling factor $\alpha/r=2$. For PiSSA and MiLoRA, we follow the commonly used configuration with learning rate $1\times 10^{-5}$, rank $r=16$, and $\alpha=32$. LoRA dropout is set to $0.0$ for all methods. The clipping range is asymmetric, with a lower clip ratio of $0.2$ and an upper clip ratio of $0.28$. The prompt batch size is $128$, and the PPO mini-batch size is $32$.

In the cosine-decay setting, we keep the optimizer, warmup, weight decay, batch size, rollout configuration, and adapter configuration matched to the corresponding constant-LR run, replacing only the constant schedule with cosine decay. For our 1.5B cosine-decay runs, the initial learning rate is $1\times 10^{-5}$, warmup is $0$, and weight decay is $0.1$.

\FloatBarrier
\begin{longtable}{@{}lll@{}}
\caption{Hyperparameters for the DAPO 1.5B experiments.}
\label{tab:hyperparameters_dapo_1p5b} \\

\toprule
\textbf{Hyperparameter} & \textbf{Constant-LR setting} & \textbf{Cosine-decay setting} \\
\midrule
\endfirsthead

\multicolumn{3}{c}%
{{\tablename\ \thetable{} }} \\
\toprule
\textbf{Hyperparameter} & \textbf{Constant-LR setting} & \textbf{Cosine-decay setting} \\
\midrule
\endhead

\midrule
\multicolumn{3}{r}{{}} \\
\endfoot

\bottomrule
\endlastfoot

\multicolumn{3}{l}{\textit{Model and Software}} \\
Base model & DeepSeek-R1-Distill-Qwen-1.5B & DeepSeek-R1-Distill-Qwen-1.5B \\
Training framework & verl 0.7.0.dev & verl 0.7.0.dev \\
Inference engine & vLLM 0.11.0 & vLLM 0.11.0 \\
Flash Attention & 2.8.1 & 2.8.1 \\
PyTorch & 2.8.0+cu126 & 2.8.0+cu126 \\
Hardware & 8 $\times$ NVIDIA A100-SXM4-80GB & 8 $\times$ NVIDIA A100-SXM4-80GB \\
\midrule

\multicolumn{3}{l}{\textit{Optimization}} \\
Optimizer & AdamW & AdamW \\
Learning rate & $1\times10^{-5}$ & $1\times10^{-5}$ \\
Learning rate schedule & Constant & Cosine decay \\
Warmup steps & $0$ & $0$ \\
Weight decay & $0.1$ & $0.1$ \\
Gradient clipping & $1.0$ & $1.0$ \\
Training steps & $500$ & $500$ \\
\midrule

\multicolumn{3}{l}{\textit{Batch Size}} \\
Prompt batch size & $128$ & $128$ \\
PPO mini-batch size & $32$ & $32$ \\
Responses per prompt & $8$ & $8$ \\
\midrule

\multicolumn{3}{l}{\textit{GRPO / DAPO}} \\
Advantage estimator & GRPO & GRPO \\
KL reward coefficient & $0.0$ & $0.0$ \\
Actor KL loss coefficient & $0.0$ & $0.0$ \\
Clip ratio lower / upper & $0.2 / 0.28$ & $0.2 / 0.28$ \\
Clip ratio $c$ & $10.0$ & $10.0$ \\
Loss aggregation & Token mean & Token mean \\
Entropy coefficient & $0$ & $0$ \\
Overlong buffer length & $4096$ & $4096$ \\
Overlong penalty factor & $1.0$ & $1.0$ \\
\midrule

\multicolumn{3}{l}{\textit{Adapter Configuration}} \\
Adapter type & LoRA-style & LoRA-style \\
Rank ($r$) & $16$ & $16$ \\
Alpha ($\alpha$) & $32$ & $32$ \\
Target modules & All linear layers & All linear layers \\
Dropout & $0.0$ & $0.0$ \\
Bias & None & None \\
\midrule

\multicolumn{3}{l}{\textit{Training Rollout Generation}} \\
Max prompt length & $512$ & $512$ \\
Max response length & $16384$ & $16384$ \\
Temperature & $1.0$ & $1.0$ \\
Top-$p$ & $1.0$ & $1.0$ \\
Top-$k$ & $-1$ & $-1$ \\
Rollout tensor parallel size & $2$ & $2$ \\
Rollout GPU memory utilization & $0.75$ & $0.75$ \\
Chunked prefill & Enabled & Enabled \\

\end{longtable}
\FloatBarrier

\subsection{Evaluation setup}

\paragraph{Evaluation benchmarks.}
We evaluate all methods on five mathematical reasoning benchmarks:

  \begin{itemize}
      \item \textbf{GSM8K}~\citep{cobbe2021gsm8k}: Grade-school math word problems requiring multi-step arithmetic reasoning. We use the full test
  set of 1,319 examples and report pass@1 with greedy decoding.
      \item \textbf{MATH500}~\citep{hendrycks2021measuringmathematicalproblemsolving}: A 500-problem subset of the MATH benchmark covering algebra,
  geometry, counting, probability, number theory, and precalculus. We report pass@4.
      \item \textbf{AIME 2022/2023/2024}: Competition-level mathematical reasoning problems from the American Invitational Mathematics
  Examination. Each year contains 30 problems. We report pass@32.
  \end{itemize}

  \begin{table}[H]
  \centering
  \caption{Evaluation settings for the DAPO 1.5B experiments.}
  \label{tab:evaluation_settings_dapo_1p5b}
  \vspace{0.5em}
  \small
  \begin{tabular}{@{}llllll@{}}
  \toprule
  \textbf{Benchmark} & \textbf{Metric} & \textbf{Samples} & \textbf{Temperature} & \textbf{Top-$p$} & \textbf{Max response length} \\
  \midrule
  GSM8K & pass@1 & $1$ & $0$ & $1.0$ & $4096$ \\
  MATH500 & pass@4 & $4$ & $0.6$ & $0.95$ & $8192$ \\
  AIME 2022--2024 & pass@32 & $32$ & $0.6$ & $0.95$ & $16384$ \\
  \bottomrule
  \end{tabular}
  \end{table}

\paragraph{Evaluation setup and metrics.} We report pass@1 with greedy decoding for GSM8K, pass@4 with temperature sampling for MATH500, and pass@32 with temperature sampling for AIME.  Specifically, for the greedy decoding in GSM8K, we set $\tau = 0$ and top-$p = 1.0$. For the temperature sampling in MATH500 and AIME, we utilize $\tau = 0.6$ and top-$p = 0.95$. Across all benchmarks, top-$k$ is disabled ($k = -1$) and a maximum prompt length of 1024 tokens is used. To accommodate the increasing reasoning complexity and expected output lengths, the maximum response length is set to 4096 for GSM8K, 8192 for MATH500, and 16,384 for AIME.
 
For answer scoring, we use the DAPO math reward implementation provided in verl. GSM8K is scored with flexible numerical extraction to avoid undercounting correct answers that do not follow the strict \texttt{\#\#\#\#} answer prefix. For MATH500 and AIME, we use the default mathematical answer normalization and exact-match scoring implemented by the DAPO reward function.

\section{Additional Ablation Study}
\label{app:additional_experiments}

\begin{figure}[t]    
    \centering       \includegraphics[width=1\linewidth]{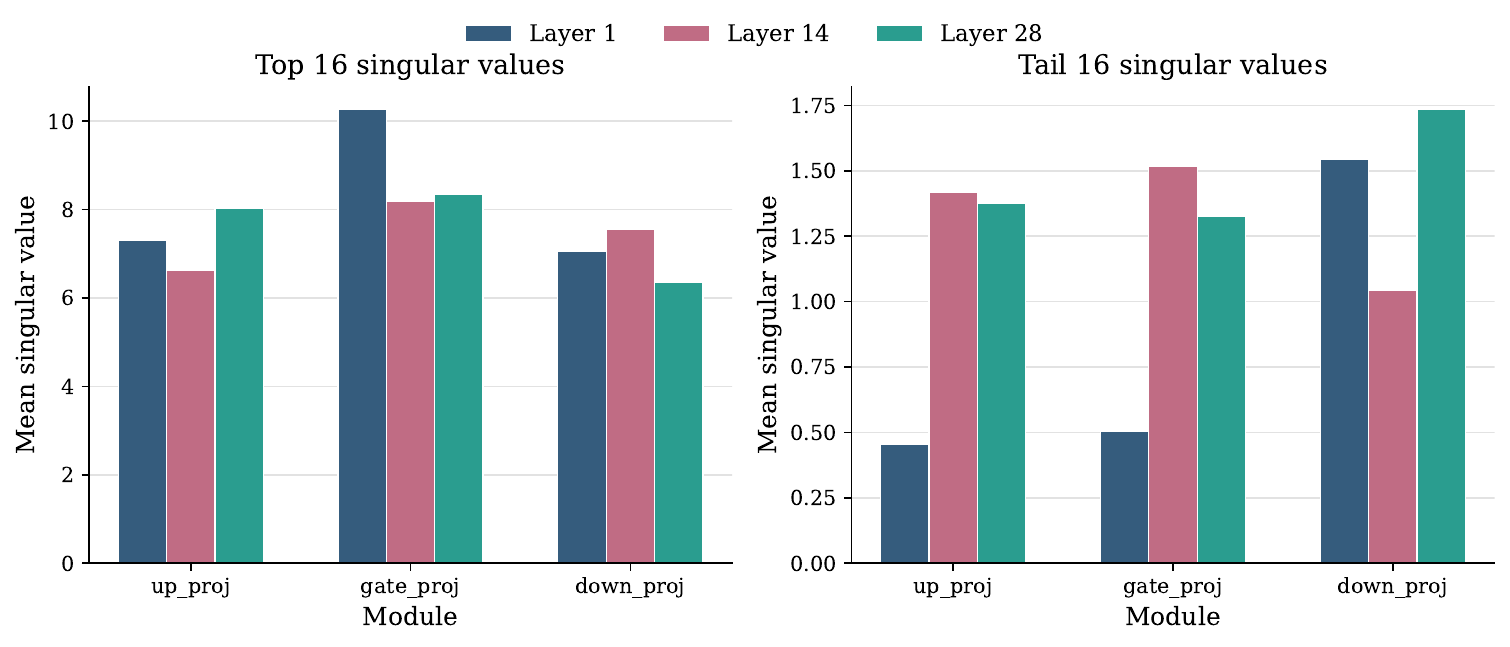}       
     \caption{Mean singular values of MLP projection weights in DeepSeek-R1-Distill-Qwen-1.5B (28
  layers total) across representative layers 1, 14, and 28, corresponding to the first,    
  middle, and last transformer layers, respectively.}
           \label{fig:singular_values}                     
\end{figure}

\subsection{PiSSA and MiLoRA failure analysis}                                
\label{app:pissa_analysis}        
\begin{table}[H]                             
\centering                                
\caption{Initialization norms (Frobenius) for LoRA and MiLoRA on DeepSeek-R1-Distill-Qwen-1.5B, averaged across all target linear layers at rank $r=16$.}         
\label{tab:init_norms}      
\vspace{0.5em}              
\small                             
\begin{tabular}{@{}lcc@{}} 
\toprule                                 
& LoRA & MiLoRA \\ 
\midrule                                  
$\|B_0 A_0\|_F$ & $0$ & $5.11$ \\ 
\bottomrule                                    
\end{tabular}                              
\end{table} 

Recent work~\citep{yin2025evaluatingparameterefficientmethods} attributes MiLoRA's failure to its near-zero initialization: the minor singular values are so small that once RL begins, the gradient flow pulls the update toward the principal components, causing a spectral collapse similar to PiSSA. We revisit this explanation on DeepSeek-R1-Distill-Qwen-1.5B and report two observations that differ from this account and may be of independent interest to readers.

First, the initial MiLoRA adapter is not negligible in magnitude. As reported in Table~\ref{tab:init_norms}, MiLoRA has $\|B_0 A_0\|_F = 5.11$, which differs from the near-zero initialization hypothesis. Moreover, standard LoRA satisfies $\|B_0 A_0\|_F = 0$ yet remains stable during RL training, suggesting that initial magnitude alone does not account for MiLoRA's behavior.

Second, the trained MiLoRA model retains non-negligible mass in the off-principal subspace. Figure~\ref{fig:singular_values} shows that although the top singular components dominate the tail across MLP projections, the tail values remain clearly non-zero. For DeepSeek-R1-Distill-Qwen-1.5B (28 layers total), we examine three representative layers ($1$, $14$, $28$) and find that the mean of the bottom-$16$ singular values ranges from $0.46$ to $1.73$, depending on projection the module. Together, these observations suggest that the tail spectrum does not appear to vanish, either at initialization or throughout training.

\subsection{Further ablations for RLVR}
\label{app:svd_cost}

\paragraph{Additional task domains and model families.}
We further evaluate on Llama 3.2-3B-Instruct and Qwen2.5-1.5B-Instruct (different model families). To assess domain generality, we additionally consider code generation: Llama 3.2-3B-Instruct trained on MBPP-style program synthesis~\citep{austin2021program} with a test-case-based reward. LoRA-RLPO remains stable and effective under this different reward structure. 

\begin{table}[H]
    \centering
    \caption{Code generation on Llama 3.2-3B-Instruct (MBPP-style, test-case reward).}
    \label{tab:code_gen_llama}
    \vspace{0.5em}
    \small
    \begin{tabular}{@{}ccc@{}}
    \toprule
    \textbf{LoRA} & \textbf{LoRA-RLPO(Ours)} & \textbf{LoRA-RLMO(Ours)} \\
    \midrule
    $43.44 \pm 3.10$ & ${45.67 \pm 1.41}$ & $\mathbf{46.11 \pm 1.26}$ \\
    \bottomrule
    \end{tabular}
\end{table}

\paragraph{Larger model size.}
We fine-tune Qwen2.5-7B-Instruct using GRPO~\citep{shao2024deepseekmathpushinglimitsmathematical} on DAPO-Math-17k~\citep{yu2026dapo}, with rank $r=32$ LoRA applied to all linear layers.

For the 7B experiments, we train for $150$ optimization steps using AdamW with a constant learning-rate schedule and no warmup. For fair comparisons, we use a learning rate of $1 \times 10^{-5}$ and scaling factor $\alpha = 64$ for standard LoRA, \RLPO, and \RLMO. For PiSSA and MiLoRA, we adopt the same learning rate of $1 \times 10^{-5}$ and $\alpha = 64$, in line with their standard tuning practices, ensuring a fair comparison across all methods. The effective batch size is $32$ ($4$ prompts per batch with $8$ responses sampled per prompt), and we use a KL penalty coefficient of $\beta = 0.001$ in the GRPO objective.

As shown in Table~\ref{tab:qwen_7b_results}, our proposed geometry-preserving initializations maintain their empirical advantages at this larger scale. Consistent with our observations on the 1.5B model, the SVD-based variants PiSSA and MiLoRA struggle under the strict KL constraints of RLVR, yielding average scores ($24.74$ and $26.72$) substantially below standard LoRA ($30.39$). In contrast, both \RLMO and \RLPO maintain stable optimization dynamics and achieve consistent improvements. Notably, \RLPO attains the highest average score of $\mathbf{35.96}$ across all mathematical reasoning benchmarks, with substantial gains on GSM8K and AIME. These results demonstrate that our theoretically motivated initializations scale robustly to 7B-parameter models without additional hyperparameter tuning.

\begin{table}[H]
\centering
\renewcommand{\arraystretch}{0.70}
\begin{tabular*}{\textwidth}{@{\extracolsep{\fill}}lccccc@{}}
\toprule
& LoRA & MiLoRA & \RLMO (Ours) & PiSSA & \RLPO (Ours) \\
\midrule
$B_0$ & $\mathbf{0}$ & $U_{-r}\Sigma_{-r}^{1/2}$ & $\mathbf{0}$ & $U_r\Sigma_{r}^{1/2}$ & $\mathbf{0}$ \\
$A_0$ & $\mathcal{N}(0, \frac{1}{n})$ & $\Sigma_{-r}^{1/2}V_{-r}^\top$ & $V_{-r}^\top$ & $\Sigma_{r}^{1/2}V_r^\top$ & $V_r^\top$ \\
\midrule
GSM8K$^{@1}$   & $79.13{\scriptstyle\pm4.6}$ & $57.27{\scriptstyle\pm12.9}$ & $74.30{\scriptstyle\pm12.3}$ & $55.60{\scriptstyle\pm44.0}$ & $\mathbf{85.29}{\scriptstyle\pm2.4}$ \\
MATH500$^{@4}$ & $52.80{\scriptstyle\pm2.3}$ & $54.13{\scriptstyle\pm1.5}$  & $\mathbf{56.73}{\scriptstyle\pm1.3}$ & $50.33{\scriptstyle\pm1.1}$ & $55.60{\scriptstyle\pm2.6}$ \\
AIME22$^{@16}$ & $3.33{\scriptstyle\pm2.7}$  & $2.22{\scriptstyle\pm1.9}$  & $7.78{\scriptstyle\pm1.6}$  & $4.44{\scriptstyle\pm1.9}$  & $\mathbf{7.78}{\scriptstyle\pm1.6}$ \\
AIME23$^{@16}$ & $11.11{\scriptstyle\pm4.2}$ & $11.11{\scriptstyle\pm1.9}$ & $10.00{\scriptstyle\pm0.0}$ & $6.67{\scriptstyle\pm3.3}$  & $\mathbf{13.33}{\scriptstyle\pm2.7}$ \\
AIME24$^{@16}$ & $5.56{\scriptstyle\pm1.6}$  & $8.89{\scriptstyle\pm1.9}$  & $13.33{\scriptstyle\pm2.7}$ & $6.67{\scriptstyle\pm3.3}$  & $\mathbf{17.78}{\scriptstyle\pm6.8}$ \\
\midrule
Avg            & $30.39$ & $26.72$ & $32.42$ & $24.74$ & $\mathbf{35.96}$ \\
\bottomrule
\end{tabular*}
\vspace{4pt}
\caption{Evaluation results of Qwen2.5-7B-Instruct fine-tuned with GRPO on DAPO-Math-17k. Our proposed \RLPO and \RLMO initializations outperform standard LoRA, whereas PiSSA and MiLoRA exhibit performance degradation.}
\label{tab:qwen_7b_results}
\end{table}

\paragraph{Learning rate sensitivity.}
\label{app:lr_sensitivity}
We conduct a learning-rate sweep on Qwen2.5-7B-Instruct to evaluate the robustness of different initialization methods. Figure~\ref{fig:lr_sweep} reports the average accuracy across learning rates $\{10^{-6}, 10^{-5}, 10^{-4}\}$. \RLPO and \RLMO consistently outperform standard LoRA at every learning rate. All methods achieve peak performance at $10^{-5}$, which we adopt as the default learning rate for all experiments reported in this paper.

\begin{center}
    \includegraphics[width=0.7\linewidth]{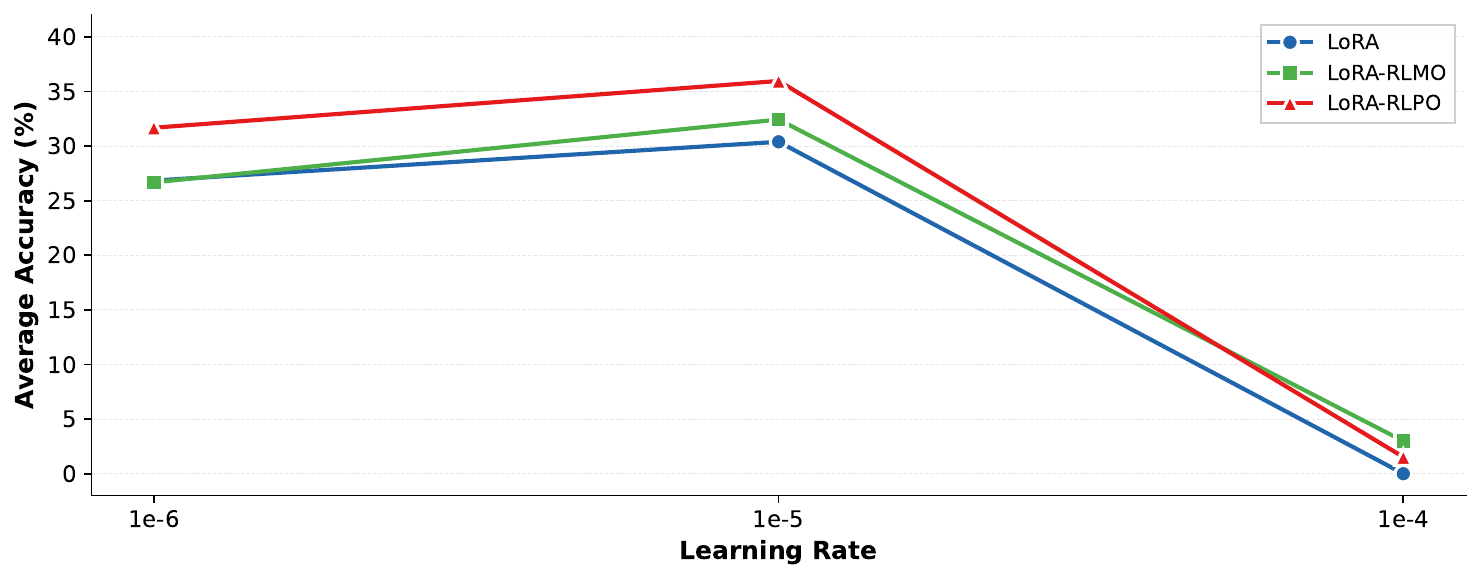}
    \captionof{figure}{Learning rate sensitivity comparison. \RLPO and \RLMO consistently outperform standard LoRA across learning rates, with peak performance at $10^{-5}$.}
    \label{fig:lr_sweep}
\end{center}

\paragraph{SVD initialization preprocessing cost.}
Table~\ref{tab:svd_cost} reports the SVD preprocessing cost across model sizes. Measurements are conducted with rank $r=32$ in bfloat16 precision on a single NVIDIA A100 GPU. Note that LoRA-RLPO and LoRA-RLMO have identical SVD costs. This preprocessing is a one-time cost incurred before training begins.

\begin{table}[H]
  \centering
  \caption{SVD preprocessing cost evaluated with $r=32$, bfloat16 precision, on a single A100 GPU.}
  \label{tab:svd_cost}
  \vspace{0.5em}
  \small
  \begin{tabular}{@{}lccc@{}}
  \toprule
  \textbf{Model} & \textbf{Parameters} & \textbf{Wall-Clock Time} & \textbf{Peak GPU Memory} \\
  \midrule
  Qwen3-4B-Instruct & 4B & 1.8 min & 8.3 GiB \\
  Qwen2.5-7B-Instruct & 7B & 3.5 min & 16.0 GiB \\
  Qwen2.5-14B-Instruct & 14B & 12.3 min & 29.8 GiB \\
  \bottomrule
  \end{tabular}
\end{table}

\subsection{Generalization to supervised fine-tuning (SFT)}
\label{sec:sft_generalization}

To investigate whether the proposed geometry-preserving initializations generalize beyond the RL fine-tuning framework, we conduct supervised fine-tuning (SFT) experiments on Qwen2.5-7B-Instruct. We evaluate across two benchmark categories:
\begin{itemize}
    \item \textbf{GLUE} (CoLA and MRPC): classification tasks evaluating linguistic acceptability and paraphrase detection.

    \item \textbf{GSM8K}: grade-school math reasoning, evaluated via strict exact match using the Language Model Evaluation Harness \citep{biderman2026lessonstrenchesreproducibleevaluation} under the standard Chain-of-Thought (CoT) prompting setup \citep{cobbe2021gsm8k, wei2022chain}.

\end{itemize}

As shown in Table~\ref{tab:sft_results}, both LoRA-RLPO and LoRA-RLMO generalize effectively to the SFT paradigm, consistently outperforming standard LoRA across all three benchmarks.

Figure~\ref{fig:sft_train_loss} shows the training loss curves across the three SFT tasks. Both LoRA-RLPO and LoRA-RLMO consistently converge faster and reach a lower final loss than standard LoRA. The improvement is most pronounced on the GSM8K reasoning task.

\begin{table}[H]
    \centering
    \caption{SFT evaluation results on Qwen2.5-7B-Instruct. Results are reported as mean $\pm$ standard deviation across 3 random seeds.}
    \label{tab:sft_results}
    \vspace{0.5em}
    \small
    \begin{tabular}{@{}lccc@{}}
    \toprule
    \textbf{Task} & \textbf{LoRA} & \textbf{LoRA-RLPO} & \textbf{LoRA-RLMO} \\
    \midrule
    CoLA (acc.) & $85.46 \pm 0.22$ & $86.42 \pm 0.24$ & $\mathbf{86.48 \pm 0.50}$ \\
    MRPC (acc.) & $86.52 \pm 0.25$ & $\mathbf{88.48 \pm 0.25}$ & $87.91 \pm 0.51$ \\
    GSM8K (strict) & $23.96 \pm 8.89$ & $29.74 \pm 0.54$ & $\mathbf{33.61 \pm 2.99}$ \\
    \bottomrule
    \end{tabular}
\end{table}

\begin{figure}[H]
    \centering
\includegraphics[width=0.9\linewidth]{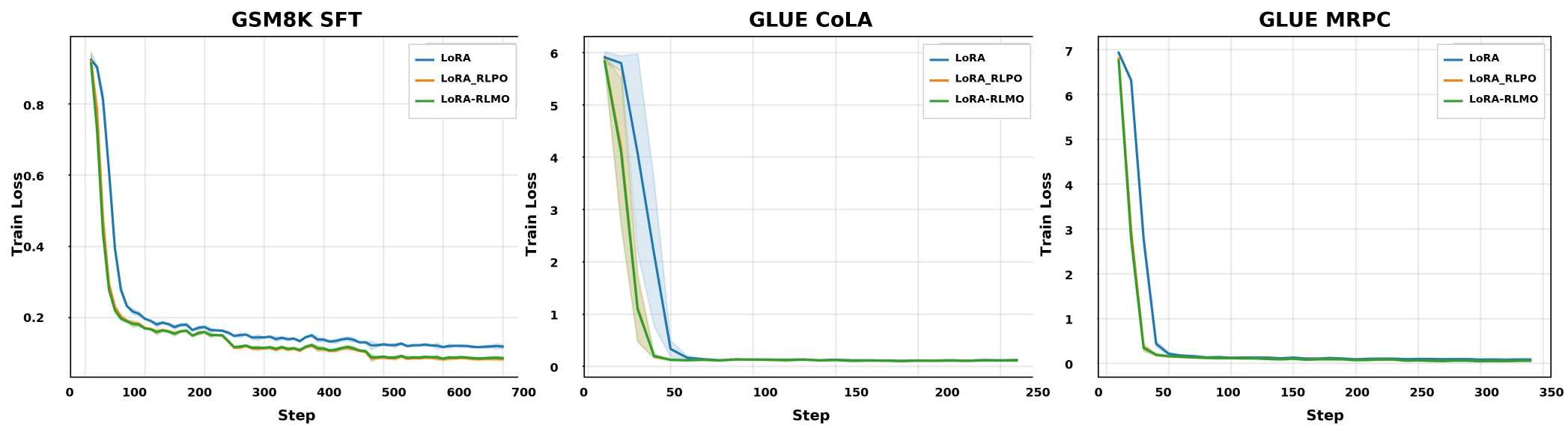} 
    \caption{Train loss curves across three SFT tasks (3 seeds, shaded region denotes $\pm 1$ standard deviation). Both LoRA-RLPO and LoRA-RLMO demonstrate faster convergence and lower final loss than standard LoRA.}
    \label{fig:sft_train_loss}
\end{figure}

\paragraph{Experimental setup.}
For all SFT experiments, we use rank $r=32$, $\alpha=64$, and a constant learning rate of $1 \times 10^{-5}$. Models are trained for $3$ epochs with a global batch size of $32$ ($4$ per device $\times\ 8$ gradient accumulation steps). All evaluations are averaged across three random seeds ($\{1, 42, 123\}$).

\end{document}